\def\ARXIVVERSION{1}
\ifdefined\ARXIVVERSION
  \documentclass[mnsc]{informs4}
  \newcommand{\arxivonly}[1]{#1}
  \def\theARTICLETOPLEFT{}
  \def\theARTICLETOPRIGHT{}
\else
  \documentclass[mnsc,blindrev]{informs4}
  \newcommand{\arxivonly}[1]{}
\fi

\OneAndAHalfSpacedXI 

\usepackage{amsmath}  
\usepackage{amssymb}  
\usepackage[colorinlistoftodos,textsize=tiny
			]{todonotes}
\usepackage{graphicx}
\usepackage{algorithm}
\usepackage{algpseudocode}
\usepackage{float}
\usepackage{mathrsfs}  
\usepackage[usenames,dvipsnames]{xcolor}
\usepackage{hyperref}
\usepackage{booktabs}
\usepackage{multirow}
\usepackage{array}
\usepackage{paralist}
\usepackage{verbatim}
\usepackage{natbib}
\usepackage{stackengine}
\usepackage{hyperref}
\hypersetup{colorlinks=true,linkcolor=teal,urlcolor=blue,citecolor=teal}
\usepackage{caption}   
\usepackage{graphicx}
\usepackage{subfigure} 
\usepackage{makecell}
\usepackage{adjustbox}   
\RequirePackage{url,natbib}
\RequirePackage{tgtermes}
\RequirePackage{newtxtext}
\RequirePackage{newtxmath}
\RequirePackage{bm}
\RequirePackage{endnotes}

\RequirePackage{tgtermes}
\RequirePackage{newtxtext}
\RequirePackage{newtxmath}
\RequirePackage{bm}
\RequirePackage{endnotes}

\usepackage{multirow}
\usepackage{booktabs}    

\OneAndAHalfSpacedXI

\usepackage{algorithm}
\usepackage{algpseudocode}
\usepackage{tikz}
\usepackage{todonotes}

\usepackage{etoolbox}     

\newcommand{\person}[2]{%
  \definecolor{person@col@#1}{named}{#2}%
  \expandafter\newcommand\csname #1\endcsname[1]{%
    \textcolor{person@col@#1}{\textbf{[#1: ##1]}}}%
}

\person{Agostino}{red}
\person{Haoxian}{purple}
\usepackage{natbib}
 \bibpunct[, ]{(}{)}{,}{a}{}{,}%
 \def\bibfont{\small}%
\EquationsNumberedThrough    

\TheoremsNumberedThrough     
\ECRepeatTheorems  %

\MANUSCRIPTNO{MNSC-0001-2024.00}

\RUNAUTHOR{Li et al.}
\RUNTITLE{Prediction-Enhanced Monte Carlo}

\TITLE{Prediction-Enhanced Monte Carlo: A Machine Learning View on Control Variate
}

\ARTICLEAUTHORS{%
\AUTHOR{Fengpei Li\thanks{equal contribution}\thanks{corresponds to: fengpei.li@morganstanley.com}, Haoxian Chen\footnotemark[1], Jiahe Lin\footnotemark[1], Arkin Gupta, Xiaowei Tan, Honglei Zhao, Gang Xu, Yuriy Nevmyvaka}
\AFF{ Morgan Stanley}

\AUTHOR{Agostino Capponi, Henry Lam}
\AFF{Department of IEOR, Columbia
University, }
}
\ABSTRACT{ 
For many complex simulation tasks spanning areas such as healthcare, engineering, and finance, Monte Carlo (MC) methods are invaluable due to their unbiased estimates and precise error quantification. Nevertheless, MC simulations often become computationally prohibitive, especially for nested, multi-level, or path-dependent evaluations lacking effective variance reduction techniques. While machine learning (ML) surrogates appear as natural alternatives, naïve replacements typically introduce unquantifiable biases. We address this challenge by introducing Prediction-Enhanced Monte Carlo (PEMC), a framework that leverages modern ML models as \emph{learned predictors}, using cheap and parallelizable simulations as \emph{features}, to output unbiased evaluations with reduced variance and runtime. As a result, PEMC eliminates the closed-form-mean requirement that constrains classical control-variate methods, preserves unbiasedness and explicit confidence intervals of MC, and achieves scheme-wide variance reduction. Our theoretical analysis quantifies the optimal allocation between expensive evaluations and large batches of cheap, parallelizable feature draws. We apply our framework to three representative settings, swap pricing under stochastic-local-volatility, swaption pricing under Heath–Jarrow–Morton models, and ambulance diversion policy, showing that PEMC reduces root-mean-squared error by 30–55\% relative to standard MC at similar computational cost. 

}


\begin{document}
\maketitle
\section{Introduction}\label{sec:intro}

Monte Carlo (MC) is a fundamental tool for evaluating complex stochastic models where closed-form solutions are unavailable. Its main strengths, namely unbiasedness and rigorous uncertainty quantification through variance estimation and confidence intervals, make it essential in high-stakes, high-complexity applications such as healthcare resource allocation and financial risk management. 
For example, transparent and bias-aware methods in healthcare are vital to ensure equitable outcomes and prevent disparities in care access \citep{rajkomar2018ensuringfairness,kamiran2012fairness,henrylamwsc_ambul}. Similarly, financial risk assessments rely on MC estimation to support  hedging, re-balancing decisions, and valid error bounds, as minor biases in financial simulations can accumulate \citep{broadie1996estimating}.

Despite these advantages, classical MC faces significant challenges in nested, multilevel, or path-dependent simulations that are difficult to parallelize. Complex stochastic dynamics can make each sample path extremely costly to generate, sometimes taking hours for one sample (e.g., to estimate credit valuation adjustments in over-the-counter contracts, see \cite {green2015xva,andersen2017xva}), and effective variance reduction methods are often unavailable \citep{asmussen2007stochastic}. Consequently, MC's inherent $\mathcal{O}(1/\sqrt{n})$ convergence rate often prohibits real-time decision-making with adequate precision. This is the case, for example, in the context of   pricing of exotic options, where sample paths need to be generated sequentially under complex dynamics, making parallelization not easily achievable \citep{glasserman2013monte}. 
These computational challenges are critical in financial derivative pricing, as  highlighted in the literature, for example, in \citet[Chapter 21]{hull2018options} and \citet{broadie1996estimating}. Similar computational bottlenecks arise in healthcare simulation, where modeling complex patient flows and hospital operations can require extensive computation time but transparent and bias-aware methods are crucial to ensure equitable outcomes \citep{rajkomar2018ensuringfairness,kamiran2012fairness,henrylamwsc_ambul}.

Conversely, machine learning (ML) based predictive modeling has emerged as a fast alternative for approximating MC evaluations, including pricing financial derivatives  \citep{hutchinson1994nonparametric,ferguson2018deep, bayer2019deep}. Neural networks and other ML techniques can learn complex mappings and, once trained, can generate predictions at negligible marginal cost - what practitioners often call ``fast evaluation" \citep{horvath2021deep}. However, their black-box nature and lack of inherent statistical control hinder the quantification of errors, rendering them unsuitable for tasks requiring rigorous reliability assessments, such as high-stakes applications in risk management \citep{hutchinson1994nonparametric,buehler2019deep}.  Thus, MC and ML represent opposite ends of the spectrum: MC is reliable and uncertainty quantifiable but intrinsically slow, while ML is computationally fast but falls short in statistical guarantees.

 Given this trade-off between the speed of ML and the reliability of MC, a natural question arises: can we design a method that combines both advantages while avoiding their respective pitfalls? In this paper, we provide a positive answer by introducing the \emph{Prediction-Enhanced Monte Carlo (PEMC)} framework. PEMC leverages predictive ML modeling to replace a large portion of costly simulations with cheap parallelizable samples from ML models, reducing variance at a cost-aware scheme-wide level. Moreover, despite relying on black-box predictive models, PEMC maintains unbiasedness, balancing computational savings with rigorous statistical guarantees.

\subsection{Our Contributions}

PEMC innovates in two aspects. First, PEMC integrates ML predictive models into any MC baseline while preserving {unbiasedness} and {error quantification} of MC. Specifically, we show how the predictive modeling of ML combined with cheap, parallelizable samples will result in a reduction of variance and costly samples, enhancing real-time efficiency for complex MC simulations. Second, PEMC can also be viewed as a modernized variant of the control variate (CV) technique, a classical approach that reduces MC variance via auxiliary outputs whose known mean offsets random fluctuations \citep{glasserman2013monte, asmussen2007stochastic,nelson1990control}. Unfortunately, the classical CV approach requires a closed-form expression for the mean of a correlated auxiliary output, which is often unavailable in many modern applications or complex stochastic systems, effectively barring the use of CV. 

To explain concretely the second point above, our PEMC framework bridges ML with MC by substantially enlarging the scope of CV to many ML candidates. The key idea is the following: In the traditional CV framework, the known mean of CV is used to achieve a per-replication variance reduction while maintaining unbiasedness, i.e., the variance is reduced even if we run just one simulation replication. While this is a powerful feature, it comes at the cost of significantly limiting CV’s applicability, due to the requirement of a closed-form expression for the mean. However, if the CV itself is computationally cheap in the sense of being estimatable with fast simulation, then we can afford to run plenty of additional MC \emph{separately} from the original MC to estimate its mean. This would preserve unbiasedness and reduce overall estimation variance relative to the total computation cost, as long as we consider efficiency improvement at the \emph{scheme-wide} level. That is, in the PEMC framework, we consider total cost-aware instead of per-replication variance reduction, through which we relax the crucial requirement of known CV mean, and this in turn allows us to substantially expand the CV methodology to cover modern ML models. The alignment of our framework with the use of ML predictions prompts us to regard PEMC as a modernized view of CV.



\subsection{Literature Review}

The literature on enhancing MC simulations with ML approximators is extensive, yet many existing ML-based methods for CV and variance reduction impose restrictive conditions that appear to limit their practical applicability. Approaches such as reproducing Stein kernels \citep{oates2017}, neural CV derived from Stein's identity \citep{wan2020neural,lam2023doubly}, regularized least squares for CV construction \citep{portier2018,south2022}, adaptive CV schemes \citep{henderson2002,henderson2004,kim2007}, and $\mathscr{L}^2$ function approximation frameworks \citep{maire2003} have demonstrated theoretical promise with potentially rapid convergence rates. Recent theoretical work by \cite{blanchet2024can} provides important insights into when regression-adjusted CVs can achieve optimal performance, showing that their effectiveness depends critically on function smoothness. However, these methods typically require knowledge of the CV mean, in turn restricting the classes of usable CV, or impose specific structural assumptions that can be difficult to check. For example, in applying neural CVs mentioned above, one faces the dilemma of either imposing strong assumptions such as Stein's identity or score functions to guarantee zero mean (\cite{belomestny2024theoretical,oh2025control,wan2020neural}) or through normalization-based decomposition (\cite{muller2020neural}), which in turn constrains the CV to restrictive functional forms, or else forfeiting the unbiasedness guarantee. 
As a result, while methods like neural CVs offer strong theoretical foundations, their applicability in complex modeling scenarios can be undermined. PEMC avoids this trade-off through the scheme-wide, cost-aware view introduced in Section 1.1, which relaxes the mean-knowledge requirement without restricting the functional form of the predictor.



In a different line of literature, the works of \cite{emsermann2002improving} on quasi CVs and \cite{pasupathy2012control}  on CVs using Estimated Means (CVEMs) is related to our approach, because they rely on an estimated mean (rather than a priori known, and thus not a traditional CV) of a quantity that is expected to exhibit a strong correlation with the random variable of interest. There are noticeable differences between our methodology and theirs.
To begin with, our primary objective is to integrate modern ML capabilities into MC, so that our estimates can retain the error-quantifiability of the latter approach while leveraging the strengths of ML. To this end, a core component of our methodology regards what is needed, and how to train a suitable CV, rather than relying on a prefixed CV candidate, which may not always be obvious or even available. Moreover, our investigation includes the use of a pre-training approach in place of the adaptive algorithms employed in previous works, and provides a framework applicable to both finite-sample and asymptotic regimes, unlike the exclusively asymptotic focus of prior studies. 


Within the literature of MC variance reduction, PEMC also relates to Multilevel Monte Carlo (MLMC) methods, originally introduced and popularized by \cite{giles2015multilevel}, and the idea of multi-fidelity modeling popular in scientific computation \citep{peherstorfer2018survey}. MLMC reduces variance by leveraging a hierarchy of discretizations (e.g., coarser and finer time steps in SDE simulations) and coupling them to achieve computational efficiency over naive Monte Carlo approaches \citep{giles2008multilevel}. The key idea is to couple simulations across resolution levels so that their differences exhibit reduced variance—or even enable unbiased estimators \citep{rhee2015unbiased, rychlik1990unbiased}. Along a similar vein, in multi-fidelity modeling, computational simulation models of different levels of resolution, which in turn bring in different accuracies and computational demands, are jointly utilized and concatenated, with input model parameters treated as random variables that allow coupling between the model levels \citep{peherstorfer2018survey}.
Note that, despite the strong theoretical foundations, MLMC can be difficult to implement in practice, since constructing effective couplings, ensuring numerical stability, and tuning parameters to realize the theoretical complexity benefits are nontrivial tasks \citep{prescott2020multifidelity}. PEMC adopts the coupling paradigm of MLMC and multi-fidelity modeling but innovates by using ML predictors as CV.

Outside the MC literature, our PEMC framework aligns most closely with the recently proposed Prediction-Powered Inference (PPI) approach \citep{angelopoulos2023prediction}. PPI is a statistical framework that enables researchers to construct valid confidence intervals and p-values when combining a small dataset with gold-standard labels and a large dataset with ML predictions. The key challenge PPI addresses is how to leverage abundant but potentially biased ML predictions alongside scarce but trustworthy ground-truth data for statistical inference. Such PPI-type hybrid strategy is becoming increasingly important in evaluating generative models such as large language models (LLMs) \citep{zrnic2024cross,boyeau2024autoeval,eyre2024auto}, where sparse human annotations is combined with abundant ML-generated labels to facilitate accurate model performance evaluation even in low-label scenarios. 


Finally, our work is also related to a branch of literature in computational finance, which has focused on  harnessing ML to enhance the speed and accuracy of derivatives pricing. By using techniques such as neural network surrogates, these methods can approximate pricing functions or the Greeks at a fraction of the computational cost of traditional numerical schemes \citep{hutchinson1994nonparametric,buehler2019deep,sirignano2018dgm,rufwang2021review}. Early work by \citet{hutchinson1994nonparametric}, for instance, uses neural networks to relax stringent modeling assumptions, while more recent approaches \citep{buehler2019deep,sirignano2018dgm,rufwang2021review} demonstrate that deep architectures can effectively handle high-dimensional and path-dependent pricing problems. However, these purely data-driven approximations typically lack the rigorous error controls afforded by more classical methods. In the context of  option pricing, even though both MC and partial differential equation (PDE)-based methods are widely used, it is known that PDE methods generally cannot be used to price path-dependent
 options (i.e., payoff depends on the price path). While special cases allow for PDE formulations through state-space augmentation or boundary condition adjustments, these remain as exceptions rather than the rule \citep{sabate2021unbiased}.


The rest of the paper is organized as follows. In Section \ref{sec:framework}, we motivate and introduce PEMC. In Section \ref{intopemc}, we provide more details on our framework via a running example of Asian option pricing. In Section \ref{theoryresult},
we present theoretical results and their practical implications. In Section \ref{apps}, we apply PEMC to production-grade exotic option pricing, including variance swaps pricing under stochastic (local) volatility models and the pricing of swaptions under the HJM framework, as well as an ambulance diversion problem. We conclude the paper in Section \ref{conclude} with discussions on broader implications and potential directions for future research.

\section{The Prediction-Enhanced Monte Carlo Framework}\label{sec:framework}
We provide an overview and motivational explanation of our PEMC framework.
Suppose we want to produce an unbiased evaluation of the quantity
\begin{equation}\label{icmlmain}
    \mu(\boldsymbol{\theta})=\mathbb E_{\boldsymbol{\theta}}[f_{\boldsymbol{\theta}}(\boldsymbol{Y})]
\end{equation}
for arbitrary, potentially high-dimensional parameter $\boldsymbol{\theta}\in\boldsymbol\Theta\subseteq\mathbb R^{d_\theta}$, where $\Theta$ is the parameter space. Here, $\boldsymbol{Y}\in\mathbb R^{d_y}$ represents a high-dimensional random vector whose probability distribution is parameterized by $\mathbb E_{\boldsymbol{\theta}}$ and $f_{\boldsymbol{\theta}}$ is the evaluation function $\mathbb R^{d_y}\to\mathbb R$. In general, both $f_{\boldsymbol{\theta}}$ and $\mathbb E_{\boldsymbol{\theta}}$ can depend on $\boldsymbol{\theta}$. The goal is to produce an unbiased estimate of $\mu(\boldsymbol{\theta})$ for any $\boldsymbol{\theta}\in\boldsymbol\Theta$, and the primary bottleneck stems from costly simulation of $\boldsymbol{Y}$.

\subsection{Existing Challenges of MC, ML, and CV}

 When $\boldsymbol Y$ can be simulated and $f_{\boldsymbol{\theta}}$ evaluated, then standard MC amounts to generating many $f_{\boldsymbol{\theta}}(\boldsymbol Y_i)$, $i=1,\ldots,n$ and outputting their average. This approach is unbiased and its statistical error can be quantified straightforwardly via standard confidence intervals. However, it is demanding when evaluating $f_{\boldsymbol{\theta}}$ and simulating $\boldsymbol Y$ is expensive. 
 
 To reduce computation, one can approximate $\mu(\boldsymbol{\theta})$ with a ML regression model, say $\hat\mu(\boldsymbol{\theta})$, which can be evaluated more cheaply than simulating many $f_{\boldsymbol{\theta}}(\boldsymbol Y_i)$'s. This natural approach, however, would face challenges at multiple levels when reliability and error quantification is important. From a theoretical standpoint, ML generalization bounds typically assess training errors only within the model class, i.e., $\hat\mu(\boldsymbol{\theta})-\mu^*(\boldsymbol{\theta})$ where $\mu^*(\cdot)$ is the best in-class model that can be different from $\mu(\cdot)$. Moreover, even assuming $\mu^*(\cdot)$ is a ``rich" class, these bounds still only assess errors in an aggregate, worst-case fashion, i.e., high-probability bounds treating $\boldsymbol{\theta}$ as a random covariate. On the practical front, training-testing split or cross-validation succumbs to similar issues, in that it cannot measure bias against $\mu(\cdot)$, nor at the level of each $\boldsymbol{\theta}$. That is, in terms of reliability and uncertainty handling, MC appears superior by providing accurate, straightforward and instance-specific error assessment on each estimate of $\mu(\boldsymbol{\theta})$. None of these advantages are available for ML despite its computational edge.

 To leverage the respective advantages of MC and ML, one can use CV which, as mentioned in Section~\ref{sec:intro}, is a classical variance reduction method. CV aims to reduce the error of each MC run, consequently the required replication size, by utilizing information from ``auxiliary'' simulation outputs. Specifically, in generating $f_{\boldsymbol{\theta}}(\boldsymbol{Y}_i)$, suppose we can also generate auxiliary output $\boldsymbol{X}_i$ and evaluate $g(\boldsymbol{\theta},\boldsymbol{X}_i)$ for some function $g$. Then we can use $f_{\boldsymbol{\theta}}(\boldsymbol{Y}_i)-g(\boldsymbol{\theta},\boldsymbol{X}_i)+\mathbb E[g(\boldsymbol{\theta},\boldsymbol{X})]$ as an unbiased replicate for $\mathbb E_{\boldsymbol{\theta}}[f_{\boldsymbol{\theta}}(\boldsymbol{Y})]$. In particular, if $f_{\boldsymbol{\theta}}(\boldsymbol{Y})$ and $g(\boldsymbol{\theta},\boldsymbol{X})$ are highly correlated, this replicate will have a smaller variance than $f_{\boldsymbol{\theta}}(\boldsymbol{Y}_i)$. 

Naturally, one can think of using ML to train a good predictor $g(\boldsymbol{\theta},\boldsymbol{X})$ for $f_{\boldsymbol{\theta}}(\boldsymbol{Y})$, which would be highly correlated with $f_{\boldsymbol{\theta}}(\boldsymbol{Y})$ and serve as an effective CV. However, taking aside the question of how to construct $\boldsymbol{X}$ and $g$, the CV method requires knowledge of $\mathbb E_{\boldsymbol{\theta}}[g(\boldsymbol{\theta},\boldsymbol{X})]$, the closed-from mean on the auxiliary output. This knowledge is a key requirement for CV to achieve variance reduction. Yet, it is seldom available, which has severely limited the applicability of CV in practice \citep{glasserman2013monte}.

\subsection{Our Remedy and Requirements}\label{pemc overview}
Our PEMC aims to leverage the CV idea, with a twist to circumvent the need on the closed-form mean knowledge in order to allow for effective integration of MC with modern ML. This requires crucially a scheme-wide cost-aware view on variance reduction. Concretely, PEMC aims to estimate \eqref{icmlmain} by
\begin{equation}\label{overviewpemc}
    \hat{\mu}_{PEMC}(\boldsymbol{\theta})
    :=\frac{1}{n}\sum_{i=1}^n\big(f_{\boldsymbol{\theta}}(\boldsymbol{Y}_i)- g(\boldsymbol{\theta},\boldsymbol{X}_i)\big)+\frac{1}{N}\sum_{j=1}^N g(\boldsymbol{\theta},\boldsymbol{\tilde X}_j).
\end{equation}

Here, $\boldsymbol{X}$ is an auxiliary variable that is dependent on  $\boldsymbol{Y}$. Suppose, for now,  that we are given the $g$ function and the $\boldsymbol{X}$ variable. In \eqref{overviewpemc}, each pair of $(\boldsymbol{X}_i, \boldsymbol{Y}_i)$ is coupled and simulated together. On the other hand, each $\tilde{\boldsymbol{X}}_j$ is simulated independently from all $(\boldsymbol{X}_i, \boldsymbol{Y}_i)$'s, and has the same marginal distribution as $\boldsymbol{X}$. In this case, it is straightforward to see that \eqref{overviewpemc} is an unbiased estimator of \eqref{icmlmain}. On the other hand, in terms of cost efficiency, the total simulation cost consists of
\begin{equation}
[\text{cost of generating each\ }f_{\boldsymbol{\theta}}(\boldsymbol{Y}_i)-g(\boldsymbol{\theta},\boldsymbol{X}_i)]\times n+ [\text{cost of generating each\ }g(\boldsymbol{\theta},\tilde{\boldsymbol{X}}_j)]\times N\label{cost rough}
\end{equation}
while the total variance is
\begin{equation}
\frac{Var(f_{\boldsymbol{\theta}}(\boldsymbol{Y}_i)-g(\boldsymbol{\theta},\boldsymbol{X}_i))}{n}+\frac{Var(g(\boldsymbol{\theta},\tilde{\boldsymbol{X}}_j))}{N}\label{var rough}
\end{equation}

Suppose that $g$ and $\boldsymbol{X}$ are designed such that the cost of generating each $g(\boldsymbol{\theta},\tilde{\boldsymbol{X}}_j)$ is relatively negligible compared to each original MC target replicate $f_{\boldsymbol{\theta}}(\boldsymbol{Y}_i)$. Then, the total cost \eqref{cost rough} will become approximately $[\text{cost of evaluating each\ }f_{\boldsymbol{\theta}}(\boldsymbol{Y}_i)-g(\boldsymbol{\theta},\boldsymbol{X}_i)]\times n$, while the total variance \eqref{var rough} will become approximately $Var(f_{\boldsymbol{\theta}}(\boldsymbol{Y}_i)-g(\boldsymbol{\theta},\boldsymbol{X}_i))/n$. In this case, the cost-variance structure is the same as the conventional CV method: If $g(\boldsymbol{\theta},\boldsymbol{X})$ is a good predictor for $f_{\boldsymbol{\theta}}(\boldsymbol{Y})$, then $Var(f_{\boldsymbol{\theta}}(\boldsymbol{Y}_i)-g(\boldsymbol{\theta},\boldsymbol{X}_i))$ is smaller than $Var(f_{\boldsymbol{\theta}}(\boldsymbol{Y}_i))$ and thus we attain a variance reduction to the overall estimator. Like in the conventional CV method, to achieve a total reduction from $Var(f_{\boldsymbol{\theta}}(\boldsymbol{Y}_i))/n$ to $Var(f_{\boldsymbol{\theta}}(\boldsymbol{Y}_i)-g(\boldsymbol{\theta},\boldsymbol{X}_i))/n$, we generate $n$ samples of $f_{\boldsymbol{\theta}}(\boldsymbol{Y}_i)-g(\boldsymbol{\theta},\boldsymbol{X}_i)$, i.e., variance reduction is achieved per replication. 




While the above depicts the idealized situation where we assume away the cost of generating $g(\boldsymbol{\theta},\tilde{\boldsymbol{X}}_j)$, it provides guidance on what we need practically to elicit a cost-variance structure in \eqref{overviewpemc}, as revealed by \eqref{cost rough} and \eqref{var rough}, that is more favorable than naive MC. First, we need $g(\boldsymbol{\theta},\boldsymbol{X})$ to serve as a good predictor for $f_{\boldsymbol{\theta}}(\boldsymbol{Y})$, but also that $\boldsymbol{X}$ has a marginal distribution that is both efficiently simulatable and highly parallelizable, allowing a large number $N$ of independent samples of $\tilde{\boldsymbol{X}}_j$ with the same marginal as $\boldsymbol{X}$ to be generated at low cost. To this end, the selection of an appropriate $\boldsymbol{X}$ is a pivotal component of PEMC, analogous to feature selection in common ML tasks but with the additional need of being an efficiently simulatable, parallelizable random variable. Finding a good $\boldsymbol{X}$ often requires domain knowledge. In particular, given the whole path $\boldsymbol{Y}$, one can choose  $\boldsymbol{X}:=\phi(\boldsymbol{Y})$ to be some low-dimensional transformation $\phi$ of $\boldsymbol{Y}$ (e.g., ingredients for constructing $\boldsymbol{Y}$). For instance, for stochastic diferential equations (SDEs), one can define $\boldsymbol{X}$ as the sum of driving Brownian increments that generate $\boldsymbol{Y}$. This choice could capture a substantial portion of the SDE's variance, making $\boldsymbol{X}$ a good feature for predicting $f(\boldsymbol{Y})$, while its Gaussianity allows it to be generated cheaply and in parallel. 

Second, we require a function $g$ to convert $\boldsymbol{\theta}$ and $\boldsymbol{X}$ into a good predictor for $f_{\boldsymbol{\theta}}(\boldsymbol{Y})$. In this work, $g$ is pre-trained using a squared-error loss, with MC sample drawn under $(\boldsymbol{X}, \boldsymbol{Y})\sim\mathbb P_{\boldsymbol{\theta}}$ for various $\boldsymbol{\theta}\sim\boldsymbol{\Theta}$, which aims to obtain a function
$g\bigl(\boldsymbol{\theta}, \boldsymbol{X}\bigr)\approx\mathbb{E}_{\boldsymbol{\theta}}\bigl[
f_{\boldsymbol{\theta}}\bigl(\boldsymbol{Y}\bigr)\big|
\boldsymbol{\theta}, \boldsymbol{X}
\bigr]$. Note that this pre-training can be resource-consuming and takes significant amount of efforts. However, once this pre-training is done, the function $g$ and our approach can be used to improve the efficiency in evaluating $\mu(\boldsymbol{\theta})$ for any $\boldsymbol{\theta}$. For example, in the context of option pricing, the pre-trained $g$ can be stored as part of a model library tailored to a specific class of exotic options, enabling PEMC to be applied directly whenever new parameter configurations arise.





\section{An Illustrative Case Study: Pricing Asian Options with PEMC}\label{intopemc}

In this section, we employ Asian option pricing as a running example to explain the detailed steps and their implications in PEMC. Asian options are path-dependent derivatives whose payoffs depend on the average price of the underlying asset over a specified period. As mentioned in \cite{boyle1997monte} and \cite{dingec2015variance}, closed-form solutions for pricing arithmetic average Asian options are generally not available. However, when the asset price dynamic follows a geometric Brownian motion (GBM), there exists an excellent CV candidate, using the geometric average instead of the arithmetic average, whose mean is precisely known via the Black-Scholes formula. This latter case often serves as the ``textbook" example to demonstrate the power of CV \citep{glasserman2013monte}. As an important message, our case study here illustrates how PEMC can achieve a comparable level of improvements as the ``textbook" CV, for evaluating arithmetic average Asian option price when the dynamics does \emph{not} follow GBM. In other words, PEMC is capable to extend the scope of classical CV to potentially a much wider range of problems.




\subsection{Problem Setup}

The quantity of interest is an evaluation that can be expressed via:\begin{equation}\label{main}
\text{Option Price}(\boldsymbol{\theta}) = \mathbb E_{\text{risk neutral measure}(\boldsymbol{\theta})} [f_{\text{payoff}(\boldsymbol{\theta})}(\boldsymbol{Y})],
\end{equation}
where the risk neutral measure is a probability measure where all financial securities are valued as if investors are indifferent to risk, and thus all assets are assumed to earn the risk-free rate \citep{cvitanic2004introduction}.\footnote{Under this measure, prices of financial securities are the discounted expected values of their future payoffs.}  For Asian option, we can specify the framework as:\begin{enumerate}
    \item {Model Parameters} $\boldsymbol{\theta}_{\text{model}}$.  These are the parameters that specify the simulation or the stochastic model of the underlying process $\boldsymbol{Y}$ (e.g., asset value time series). For instance, in a Heston model \citep{heston1993closed}, the price process $S_t$ and the instantaneous volatility process $\nu_t$ are jointly modeled:
\begin{align}\label{heston}  
dS_t =& r S_tdt + \sqrt{\nu}_tS_tdW_t^S\nonumber\\
    d{\nu}_t = & \kappa(\eta-\nu_t)dt+\delta \sqrt{\nu_t}dW_t^\nu
\end{align}
The model parameter $\boldsymbol{\theta}_{\text{model}} := (r,\eta,\delta,\rho,\kappa)\in\mathbb R^5$ is thus 5 dimensional.   It specifies the risk neutral drift $r$, the long-run average variance $\eta$, the volatility of volatility $\delta$, the correlation $\rho$ between $W^S$ and $W^\nu$, and the mean reversion rate $\kappa$. In option pricing, these parameters are typically calibrated to fit market data and then used to generate sample paths $\boldsymbol{Y}:=\{S_t\}_{t}$. Thus $\boldsymbol{\theta}_{\text{model}}$ governs and describes the risk neutral measure that generates $\boldsymbol{Y}$.
    
    \item {Simulation Parameters} $\boldsymbol{\theta}_{\text{simulation}}$. These parameters also specify the risk neutral measure. The distinction between $\boldsymbol{\theta}_{\text{simulation}}$ and $\boldsymbol{\theta}_{\text{model}}$  primarily stems from the convention that not all simulation parameters are calibrated externally, but rather serve as hyperparameters of the MC simulation itself (either customized or directly observed). For example, when one simulates the Heston model above using Euler discretization scheme,\footnote{As mentioned in \cite{dingec2015variance}, Euler scheme was the primary method for simulation of Heston model until the exact simulation of \cite{broadie2006exact} is developed. However, due to its significant computational demands, the exact method is less commonly used in practice and simpler, biased simulation method such as Quadratic Exponential \citep{andersen2008simple} or Euler scheme remains prevalent (see, e.g., \cite{lord2010comparison}).} we use $\boldsymbol{\theta}_{\text{simulation}}$ to specify the initial stock price $S_0$, the initial volatility $\nu_0$, the time horizon $T$, and discretization time step $\Delta t$. The choice of these parameters can impact both the accuracy and efficiency of the simulation. While not typically part of the model calibration process, these parameters also play a crucial role in describing the risk-neutral measure.
    
    \item {Payoff Function Parameters} $\boldsymbol{\theta_{\text{payoff}}}$: This parameter $\boldsymbol{\theta_{\text{payoff}}}$ parametrizes the \textit{payoff function} $f_{\boldsymbol\theta_{\text{payoff}}}$. In option pricing, this specifies the details of the contract. For example, given the full price path $\boldsymbol{Y}:=\{S_t\}_t$, the payoff function for an arithmetic Asian options takes the form   
    \begin{equation}
        f_{\boldsymbol\theta_{\text{payoff}}}(\{S_t\}_t) = \max\left(\frac{1}{n_D}\sum_{i=1}^{n_D} S_{t_i} - K, 0\right)\label{asian payoff}
    \end{equation}
    where $\boldsymbol{\theta_{\text{payoff}}}$ specifies the strike level $K$, the sampling frequency $n_D$ and the observation dates $\{t_i\}_{i\in[n_D]}$.\footnote{Notice that we have  omitted the discount factor $e^{-rT}$.} 
   The discrete sampling defined by the payoff does not necessarily determine the Euler discretization in time and, in fact, the latter is usually much finer for accuracy.

\end{enumerate}

Having specified the distinct roles of  $\boldsymbol{\theta}_{\text{model}}$, $\boldsymbol{\theta}_{\text{simulation}}$ and $\boldsymbol{\theta}_{\text{payoff}}$, we will henceforth refer to them collectively as $\boldsymbol{\theta}$ when no ambiguity arises.

\subsection{Feature Engineering to Design $X$}
 We now follow the general ideas laid out in Section \ref{pemc overview}, starting with the creation of the variable $\boldsymbol{X}$ to couple with $\boldsymbol{Y}$. 
As described in Section \ref{pemc overview}, we want to find input feature $\boldsymbol{X}:=\phi(\boldsymbol{Y})$ that is some low-dimensional transformation $\phi$ of $\boldsymbol{Y}$ with:
\begin{itemize}
    \item Property 1: $\boldsymbol{X}$ can be regarded as a reasonable predictor for $f_{\boldsymbol\theta_{\text{payoff}}}(\boldsymbol{Y})$. 
     \item Property 2: In contrast to that of $\boldsymbol{Y}$, the marginal distribution of $\boldsymbol{X}$ allows for efficient and parallelizable simulation. 
    \end{itemize}

To this end, we take $\boldsymbol{X}$ as the sum of two correlated Brownian increments driving \eqref{heston}, motivated by two key properties:
\begin{itemize}
    \item Predictability: The sum of Brownian increments explains away a portion of the variance of the SDE.
    \item Marginal simulation: The sum of two Brownian increments follows a correlated Gaussian distribution, which can be generated efficiently.
\end{itemize}

More precisely, when sampling $\boldsymbol{Y}$, we collect $\boldsymbol{X}$ to be a two-dimensional vector consisting of
$$W_T^S:=\sum_{j=1}^{T/\Delta t}\Delta W^S_{j\Delta t} \quad\text { and } \quad W_T^\nu:=\sum_{j=1}^{T/\Delta t}\Delta W^\nu_{j\Delta t},$$ simulated during each step of the Euler scheme with discretization scale $\Delta t$. The marginal of $\boldsymbol{X}$ is simply a two-dimensional Gaussian: $\frac{\boldsymbol{X}}{\sqrt{T}}\sim \mathcal{N}\Big(\mathbf{0},[\begin{smallmatrix} 1 & \rho  \\ \rho  & 1 \end{smallmatrix}]\Big)$.

 
We summarize a short set of practical guidelines, aimed at helping practitioners decide a cost-effective $\boldsymbol X$ with potentially good predictive power:
\begin{itemize}
        \item {SDE / path simulation:} In diffusion models, a robust default is to let $\boldsymbol{X}$ summarize the driving noise such as Brownian increments, a coarser discretization of the path or, for jump-diffusions or marked point processes, jump counts and sizes if they have Poisson marginals.

        \item {Discrete-event simulation:} Many discrete-event simulators admit a nested decomposition of the form $\boldsymbol{Y}=\Phi(\tilde{\boldsymbol{X}},Z)$, where $\tilde{\boldsymbol{X}}$ is an ``outer'' random element (signifying scenario or environment) and $Z$ is the expensive ``inner'' conditional randomness.
        In this case, PEMC can use feature $\boldsymbol X$ to reflect a low-dimensional summary of ``outer" variable $\tilde{\boldsymbol X}$ together using inexpensive simulation. Common transformations, such as taking sum, taking maxima/minima, often retain simple marginal distributions, as we illustrate with an example in Section \ref{sec:EMS}.
    \end{itemize}

 However, the above guidelines are far from exhaustive. As with most ML tasks, problem-specific feature engineering and domain knowledge can often yield better design. We illustrate with one such example, the pricing of a floating strike lookback call option under the Heston model \eqref{heston}. Here, the option payoff is given by:\begin{equation*}
f(\{S_t\}_t) = \max(S_T - \min_{0 \leq t \leq T} S_t, 0)
\end{equation*}
where $S_t$ is the asset price process and $T$ is the option maturity. An atypical choice of feature could be: $$\boldsymbol{X} = (W_T^S, \min_{0\leq t\leq T} (r-\eta/2)t+\sqrt{\eta}W_t^S).$$ 
The rationale behind the construction of $\min_{0\leq t\leq T} (r-\eta/2)t+\sqrt{\eta}W_t^S$ term is grounded in the domain knowledge of the Heston model. Specifically, since $\eta$ represents the mean-reverting level of variance process $\nu_t$, we may approximate $\nu_t \approx \eta$ (i.e., assume fixed volatility) and simplify the asset price process to a GBM: $$S_t \approx S_0e^{(r-\eta/2)t+\sqrt{\eta}W^S_t} \text{ and } \min_{0 \leq t \leq T} S_t\approx S_0e^{\min_{0 \leq t \leq T}(r-\eta/2)t+\sqrt{\eta}W^S_t}.$$ This approximation leads to the expression 
$f(\{S_t\}_t) \approx 
    S_0e^{(r-\nu/2)T+\sqrt{\nu}[\boldsymbol{X}]_1}-S_0e^{[\boldsymbol{X}]_2}$, where $([\boldsymbol{X}]_1, [\boldsymbol{X}]_2)$ are the two coordinates of $\boldsymbol{X}$, which could serve as a potentially effective approximation of $f(\{S_t\}_t)$. This corresponds to Property 1 in the construction of $\boldsymbol{X}$. For Property 2, the marginal distribution of $\boldsymbol{X}$ permits efficient and parallelizable simulation. In particular, applying Girsanov’s Theorem enables one to derive a closed-form expression for the joint density of the two dimension of $\boldsymbol{X}$
    \citep{karatzas1991brownian}. This result makes it straightforward to sample these components directly. 



\subsection{Building Prediction Model}\label{sec: pretrain}

We now focus on the pre-training phase of our prediction model. The data generation and model training process can be outlined through the following key steps:

\begin{enumerate}
    \item We first define a parameter space $\boldsymbol\Theta$ that encompasses all combinations of $\boldsymbol{\theta}:=(\boldsymbol{\theta}_{\text{model}},\boldsymbol{\theta}_{\text{simulation}},\boldsymbol{\theta}_{\text{payoff}})$ that are of practical interest. This space covers the range of parameters that would likely be encountered in real-world pricing scenarios on a daily, weekly or monthly basis, depending on the model update frequency.
    
    \item We then draw samples of $\boldsymbol{\theta}_i$ from $\boldsymbol\Theta$, using uniform sampling or any measure that ensures comprehensive coverage. For each sampled parameter set, we generate one pair\footnote{One could also generate multiple pairs of $(\boldsymbol{Y}(\boldsymbol{\theta}_i),\boldsymbol{X}(\boldsymbol{\theta}_i))$ under the same sampled $\boldsymbol{\theta}_i$.} of $\boldsymbol{Y}(\boldsymbol{\theta}_i)$ and its corresponding $\boldsymbol{X}(\boldsymbol{\theta}_i) := \phi(\boldsymbol{Y}(\boldsymbol{\theta}_i))$. It is important to note that $\boldsymbol{X}$ and $\boldsymbol{Y}$ are coupled in this generation process. 

    \item For each generated pair, we compute the payoff $f_{\boldsymbol\theta_{i,\text{payoff}}}(\boldsymbol{Y}(\boldsymbol{\theta}_i))$, which serves as a \textit{label} in our training data. The corresponding \textit{feature} vector is comprised of $(\boldsymbol{\theta}_{i,\text{model}},\boldsymbol{\theta}_{i,\text{simulation}},\boldsymbol{\theta}_{i,\text{payoff}}, \boldsymbol{X}(\boldsymbol{\theta}_i))$.

    \item This process is repeated $N_{\text{train}}$ times to generate training dataset which consists of $N_{\text{train}}$ pairs of  
    \begin{align*}
        \textit{feature}_i:=&(\boldsymbol{\theta}_{i,\text{model}},\boldsymbol{\theta}_{i,\text{simulation}},\boldsymbol{\theta}_{i,\text{payoff}}, \boldsymbol{X}(\boldsymbol{\theta}_i))\nonumber\\
        \textit{label}_i:=&f_{\boldsymbol\theta_{i,\text{payoff}}}(\boldsymbol{Y}(\boldsymbol{\theta}_i))
    \end{align*}
    for $i\in[N_{\text{train}}]$. 

    \item The ML model $g$ is then trained to minimize the MSE loss and saved:
    \begin{equation}\label{erm}
    	\min_{g}\frac{1}{N_{\text{train}}}\sum_{i\in[N_{\text{train}}]}(\text{label}_i-g(\text{feature}_i))^2.
    \end{equation}
\end{enumerate}
\begin{algorithm}
\begin{algorithmic}[1]
\caption{Prediction Model Training in PEMC}\label{sddd}
\Procedure{DataGeneration}{$N_{\text{train}}, \boldsymbol{\Theta}$}
   \State Initialize empty datasets ($\textit{features}, \textit{labels}$)

    \For{$i = 1$ \textbf{to} $N_{\text{train}}$}
        \State Sample parameters $\boldsymbol{\theta}_i\sim\Theta$ uniformly
        \State Generate $\boldsymbol{X}(\boldsymbol{\theta}_i), \boldsymbol{Y}(\boldsymbol{\theta}_i)\sim\text{ risk neutral measure}(\boldsymbol{\theta}_i)$
        \State $\textit{label}_i \gets f_{\boldsymbol\theta_{i,\text{payoff}}}(\boldsymbol{Y}(\boldsymbol{\theta}_i))$
        \State $\textit{feature}_i \gets (\boldsymbol{\theta}_{i,\text{model}}, \boldsymbol{\theta}_{i,\text{simulation}}, \boldsymbol{\theta}_{i,\text{payoff}}, \boldsymbol{X}(\boldsymbol{\theta}_i))$
        \State Store $(\textit{feature}_i, \textit{label}_i)$ to $(\textit{features}, \textit{labels})$
    \EndFor
    \State \textbf{return} datasets ($\textit{features}, \textit{labels}$)
\EndProcedure
\Statex
\Procedure{Training}{$\textit{features}, \textit{labels}$}
    \State Initialize model $g$
    \State $N_{\text{train}} \gets \text{length}(\textit{features})$
    \State Minimize: $\min_{g}\frac{1}{N_{\text{train}}}\sum_{i\in[N_{\text{train}}]}(\text{label}_i-g(\text{feature}_i))^2$
    \State \textbf{return} trained model $g$
\EndProcedure
\end{algorithmic}
\end{algorithm}

The essential steps are summarized in Algorithm \ref{sddd}. While the data generation process can be time-intensive due to the complexity of $\boldsymbol{Y}$, it is conducted offline during the pre-training phase. This approach minimizes its impact on real-time pricing applications by effectively trading memory usage for computational speed during execution. 

This pre-training setup presents an ideal scenario for ML applications. It is important to also note that PEMC does not require $g$ to take specific forms. However, the convex and differentiable nature of the loss function (i.e., squared loss), the large supply of data (i.e., from Monte Carlo simulation), and the need for expressiveness of the model family make neural networks (NNs) an ideal choice. For training NNs, various optimization techniques and architectural enhancements—such as the Adam optimizer \citep{kingma2014adam}, batch normalization \citep{ioffe2015batch}, and skip connections \citep{he2016deep}—can be employed to improve performance. In fact, recent advancements in ML technologies have made the training process remarkably accessible and efficient. The availability of open-source, user-friendly frameworks like PyTorch \citep{paszke2019pytorch}, TensorFlow \citep{abadi2016tensorflow}, and JAX \citep{jax2018github}, coupled with GPU acceleration and highly optimized C++ backends, has dramatically simplified the process of training models on large datasets. These frameworks support recent neural network architectures, such as convolutional neural networks (CNNs) \citep{lecun1998gradient}, residual networks (ResNets) \citep{he2016deep}, and transformers \citep{vaswani2017attention}, as well as various stochastic gradient descent (SGD) optimization algorithms \citep{bottou2010large}. 
In our applications, we only utilized well-established and widely recognized NN structures that are now considered common. These models can typically be implemented using just a few lines of code, as they have been incorporated into existing packages and extensively optimized.


 \subsection{Evaluation}
Once we finish the pre-training phase, the prediction model is stored and used on-the-fly in the evaluation process as follows:
\begin{enumerate}
        \item A specific parameter $\boldsymbol{\theta}:=(\boldsymbol{\theta}_{\text{model}},\boldsymbol{\theta}_{\text{simulation}},\boldsymbol{\theta}_{\text{payoff}})\in\boldsymbol\Theta$ is given, e.g., calibrated from market data in real time, and needs to be used for pricing.    
    \item Generate $n$ pairs of $(\boldsymbol{Y}_i,\boldsymbol{X}_i)_{i \in [n]}$, from $\mathbb E_{\text{risk neutral measure}(\boldsymbol{\theta})}$. Here we have suppressed $\boldsymbol{Y}$'s dependency on $\boldsymbol{\theta}$ for convenience.
    
    \item Based on the marginal distribution of $\boldsymbol{X}$, independently generate $N$ additional samples of $(\boldsymbol{\tilde X}_j)_{j\in [N]}$ from $\mathbb E_{\text{risk neutral measure}(\boldsymbol{\theta})}$, directly according to its marginal distribution. These samples are independent of the previous data $(\boldsymbol{Y}_i,\boldsymbol{X}_i)_{i \in [n]}$.
    
     \item Utilize the pre-trained model ${g}$ to evaluate the PEMC estimator:
    \begin{equation}\label{estimator}
        \hat{\mu}_{PEMC}:=\frac{1}{n}\sum_{i=1}^n(f(\boldsymbol{Y}_i)- g(\boldsymbol{X}_i))+\frac{1}{N}\sum_{j=1}^N g(\boldsymbol{\tilde X}_j)
    \end{equation}
\end{enumerate}
For notational simplicity, we have written ${g}(\boldsymbol{X})$ instead of ${g}(\boldsymbol{\theta}_{\text{model}},\boldsymbol{\theta}_{\text{simulation}},\boldsymbol{\theta}_{\text{payoff}},\boldsymbol{X})$, and $f(\boldsymbol{Y})$ instead of $f_{\boldsymbol{\theta}_{\text{payoff}}}(\boldsymbol{Y})$. A key aspect of this procedure is the relationship between $N$ and $n$. Typically, we choose $N$ to be one, or several orders of magnitude larger than $n$. The optimal choice of $N$ versus $n$ will be discussed in Section \ref{theoryresult}. The PEMC evaluation step is summarized in Algorithm \ref{alg:pemc}. 

\begin{algorithm}[H]
\begin{algorithmic}[1]
\caption{Evaluation Procedure in PEMC}
\label{alg:pemc}
\Procedure{Evaluation}{$\boldsymbol{\theta}, n, N, g$}
   \State Generate i.i.d. $(\boldsymbol{Y}_i,\boldsymbol{X}_i)_{i \in [n]}$ from $\text{risk neutral measure}(\boldsymbol{\theta})$
   \State Generate i.i.d. $(\boldsymbol{\tilde X}_j)_{j\in [N]}$ independently from its marginal in $\text{risk neutral measure}(\boldsymbol{\theta})$
   \State Compute $\hat{\mu}_{PEMC}:= \frac{1}{n}\sum_{i=1}^n(f(\boldsymbol{Y}_i)- g(\boldsymbol{X}_i))+\frac{1}{N}\sum_{j=1}^N g(\boldsymbol{\tilde X}_j)$
   \State \textbf{return} $\hat{\mu}_{PEMC}$ 
\EndProcedure
\end{algorithmic}
\end{algorithm}

\subsection{Overall Methodology}\label{sec:illustration}
We now summarize our overall PEMC methodology applied to the Asian option under the Heston model, i.e., \eqref{asian payoff}, 
 with $\boldsymbol{\theta_{\text{payoff}}}$ specifying the strike level $K$, the sampling frequency $n_D$ and the observation dates $\{t_i\}_{i\in[n_D]}$. The Heston model is described in \eqref{heston} which requires $\boldsymbol{\theta}_{\text{model}} := (r,\eta,\delta,\rho,\kappa)\in\mathbb R^5$. Suppose we use an NN as the prediction model. Then, the PEMC framework could proceed as follows:

\begin{enumerate}
    \item Define the parameter space $\boldsymbol\Theta$ that encompasses realizations of $$\boldsymbol{\theta}:=(\underbrace{r,\eta,\delta,\rho, \kappa}_{\boldsymbol{\theta}_{\text{model}}}, \underbrace{S_0, \nu_0, \Delta t, T}_{\boldsymbol{\theta}_{\text{simulation}}},\underbrace{K, n_D, \{t_i\}_{i\in[n_D]}}_{\boldsymbol{\theta}_{\text{payoff}}} )$$ which are of practical interest. 
    
    \item Uniformly sample $\boldsymbol{\theta}$ from $\boldsymbol{\Theta}$. This step is straightforward if the parameter space is a Cartesian product of intervals. For each sampled $\boldsymbol{\theta}$, generate process $\boldsymbol{Y}:=(S_t,\nu_t)_t$ using a discretization scheme with step size $\Delta t$, with the Heston model specified by $\boldsymbol{\theta}$.
    
    \item During the sampling of $\boldsymbol{Y}$, we collect $\boldsymbol{X}$ to be the sum of Brownian increment $$W_T^S:=\sum_{j=1}^{T/\Delta t}\Delta W^S_{j\Delta t} \quad\text { and } \quad W_T^\nu:=\sum_{j=1}^{T/\Delta t}\Delta W^\nu_{j\Delta t},$$ simulated during each step of the Euler scheme. This makes the marginal of $\boldsymbol{X}$ simply a two-dimensional Gaussian: $\frac{\boldsymbol{X}}{\sqrt{T}}\sim \mathcal{N}\Big(\mathbf{0},[\begin{smallmatrix} 1 & \rho  \\ \rho  & 1 \end{smallmatrix}]\Big)$.
    
    \item Save 
    \begin{align}\label{asianopt}
        \textit{label}:=&(\frac{1}{n_D}\sum_{i=1}^{n_D} S_{t_i} - K)^+ \nonumber\\
        \textit{feature}:=&({r,\eta,\delta,\rho, \kappa}, {S_0, \nu_0, \Delta t, T},{K, n_D, \{t_i\}_{i\in[n_D]}}, W_T^S, W_T^\nu).
    \end{align}
        
    \item Repeat steps 2-4 $N_{\text{train}}$ times to generate dataset $(\textit{feature}_i,\textit{label}_i)_{i\in[N_{\text{train}}]}$ of size $N_{\text{train}}$.
    
    \item Train a NN with weights $\boldsymbol{w}$ to minimize the MSE loss:
    $$\min_{{\boldsymbol{w}}}\frac{1}{N_{\text{train}}}\sum_{i\in[N_{\text{train}}]}(\text{label}_i-NN_{\boldsymbol{w}}(\text{feature}_i))^2$$ and use $\hat w$ to 
    approximate $$NN_{\boldsymbol{\hat w}}\approx\mathbb E\bigg[(\frac{1}{n_D}\sum_{i=1}^{n_D} S_{t_i} - K)^+\bigg|{r,\eta,\delta,\rho, \kappa}, {S_0, \nu_0, \Delta t, T},{K, n_D, \{t_i\}_{i\in[n_D]}}, W_T^S, W_T^\nu\bigg].$$
    \item At evaluation, given a specific $\boldsymbol{\theta}=({r,\eta,\delta,\rho, \kappa}, {S_0, \nu_0, \Delta t, T},{K, n_D, \{t_i\}_{i\in[n_D]}}, W_T^S, W_T^\nu)\in\Theta$, we generate $n$ paired samples $(\textit{label}_i, \textit{feature}_i)_{i\in[n]}$ as in \eqref{asianopt}. We also generate $N$ i.i.d. $\sqrt{T}\mathcal{N}\Big(\mathbf{0},[\begin{smallmatrix} 1 & \rho  \\ \rho  & 1 \end{smallmatrix}]\Big)$ samples and store them as $(\tilde{\textit{feature}})_{j\in[N]}$.

    \item Output $\hat{\mu}_{PEMC}:=\frac{1}{n}\sum_{i=1}^n(\textit{label}_i- NN_{\boldsymbol{\hat{w}}}(\textit{feature}_i))+\frac{1}{N}\sum_{j=1}^NNN_{\boldsymbol{\hat{w}}}(\tilde{\textit{feature}}_j)$   \end{enumerate}

In Appendix~\ref{sec:empirical_evaluation_gbm}, we present numerical results on the performance comparisons of PEMC with standard MC and CV for Asian options under GBM. In this case, an excellent CV using the Black-Scholes formula is available under the standard CV framework to benchmark against. At the same time, in Section~\ref{apps}, we present results on more realistic and technically complex applications, including several production-grade exotic options that are vastly more complex than Asian options.


\section{Analysis and Performance}\label{theoryresult}
In this section, we provide a detailed and comprehensive analysis of the PEMC estimator. All omitted proofs are delegated to Appendix \ref{app_proofs}. Up and until Section \ref{unifiedanalysis}, all properties of PEMC discussed pertain to the evaluation stage, i.e., where $\boldsymbol{\theta}$ is arbitrary but fixed. Thus, we omit the explicit dependence on $\boldsymbol{\theta}$ in the notation (e.g., writing $\boldsymbol Y$ in place of $\boldsymbol{Y}(\boldsymbol\theta)$) until in Section \ref{unifiedanalysis} where we unify the analysis. 

\subsection{Basic Bias-Variance Analysis and Confidence Interval Construction}
We begin by establishing the unbiasedness of our estimator.
\begin{theorem}[Unbiasedness]\label{unbiased}
The PEMC estimator in \eqref{estimator} is unbiased:
$$\mathbb{E}\bigg[\frac{1}{n}\sum_{i=1}^n(f(\boldsymbol{Y}_i)- g(\boldsymbol{X}_i))+\frac{1}{N}\sum_{j=1}^N g(\boldsymbol{\tilde X}_j)\bigg] = \mathbb {E}[f(\boldsymbol{Y})]$$
\end{theorem}
\begin{proof}{Proof}
    First note that $g$ is pre-trained. The proof then follows since $\boldsymbol{X}_i$ and $\boldsymbol{\tilde X}_j$ have the same marginal distribution for any $i,j$.\Halmos
\end{proof}

Given the unbiased nature of PEMC, the analysis of PEMC as an MC method primarily hinges on its variances and computational costs.
 To facilitate this analysis, we first introduce some necessary notation.
\begin{definition}\label{defsigma}
	Given a  function $f$ that takes input $\boldsymbol{Y}$ and $g$ that takes a coupled input $\boldsymbol{X}$, we denote
\begin{equation*}
	\sigma_{f-g}^2:= \mathrm{Var}(f(\boldsymbol{Y})-g(\boldsymbol{X})), \text{ } \sigma^2_f := \mathrm{Var}(f(\boldsymbol{Y})), \text{ and } \sigma^2_g := \mathrm{Var}(g(\boldsymbol{X})).
\end{equation*}

Similarly, we denote the {cost} of generating a coupled sample $f(\boldsymbol{Y})-g(\boldsymbol{X})$ by $c_{f-g}$, the cost of generating $f(\boldsymbol{Y})$ as $c_f$, and the cost of generating $g(\tilde{\boldsymbol{X}})$ as $c_g$, where $\tilde{\boldsymbol{X}}$ is a replicate with the same marginal distribution as $\boldsymbol{X}$, but generated separately and not coupled like $(\boldsymbol{X},\boldsymbol{Y})$ in Algorithm \ref{alg:pemc}. 
\end{definition}
\begin{lemma}\label{defcov}
    Using the notations from Definition  \ref{defsigma}, we have
    \begin{equation*}
       \mathrm{Cov}(f(\boldsymbol{Y}),g(\boldsymbol{X})) = \frac{\sigma^2_f+\sigma^2_g - \sigma^2_{f-g}}{2}.
    \end{equation*}
\end{lemma}

We make two remarks here. First, we generally obtain $\boldsymbol{X}$ in the same process during the generation of $\boldsymbol{Y}$, and $g$ is low-cost. With these, we have $c_{f-g}\approx c_f$. Second, the term ``cost" is intentionally left unspecified because its interpretation varies depending on the context: In some context, cost refers to the sample size, while in others it represents the wall-clock time required to generate and evaluate all the samples, making the concept of cost instance-dependent. Therefore, for now, we use cost as an umbrella term to capture various meanings under different contexts. 
\begin{lemma}\label{pemcvariance}
	The variance of the PEMC estimator in \eqref{estimator} is 
	\begin{equation*}
		\mathrm{Var} (\hat{\mu}_{PEMC}) = \frac{1}{n}\sigma_{f-{g}}^2+\frac{1}{N}\sigma_{ g}^2.
	\end{equation*}
\end{lemma}
\begin{proof}{Proof}
    First note that $g$ is pre-trained. Moreover, in the evaluation phase, data $(\boldsymbol{Y}_i, \boldsymbol{X}_i)_{i\in[n]}$ are generated independently of $(\boldsymbol{\tilde X}_j)_{j\in[N]}$. \Halmos
\end{proof}

Having established the unbiasedness and the variance form of the PEMC estimator, we now  turn to inference. In particular, using consistent estimates of the variance, the application of the central limit theorem enables the construction of valid asymptotic confidence intervals. 

\begin{theorem}[Asymptotic Confidence Intervals]\label{thm:confidence_intervals}
In the set up of the evaluation algorithm \ref{alg:pemc}, given the prediction model $g$, $(\boldsymbol{Y}_i,\boldsymbol{X}_i)_{i\in[n]}$ and $(\boldsymbol{\tilde{X}}_j)_{j\in[N]}$, we define
\[
\hat{\sigma}^2_{f-g} = \frac{1}{n}\sum_{i=1}^n\Bigl(f(\boldsymbol{Y}_i)- g(\boldsymbol{X}_i)-
\Bigl(\frac{1}{n}\sum_{i'=1}^nf(\boldsymbol{Y}_{i'})- g(\boldsymbol{X}_{i'})
\Bigr)\Bigr)^2,
\]
and
\[
\hat{\sigma}^2_{g} = \frac{1}{N}\sum_{j=1}^N \bigl(g(\boldsymbol{\tilde{X}}_j)-\frac{1}{N}\sum_{j'=1}^Ng(\boldsymbol{\tilde{X}}_{j'})\bigr)^2,
\]
to be the respective sample variances. Let $z_{1-\alpha/2}$ denote the $(1-\alpha/2)$-quantile of the standard normal distribution. Then, the interval
$$\Bigg( \hat{\mu}_{PEMC} - z_{1-\alpha/2}\sqrt{\frac{\hat{\sigma}^2_{f-g}}{n}+\frac{\hat{\sigma}^2_{g}}{N}}, \quad \hat{\mu}_{PEMC} + z_{1-\alpha/2}\sqrt{\frac{\hat{\sigma}^2_{f-g}}{n}+\frac{\hat{\sigma}^2_{g}}{N}}\text{ }\Bigg),$$
is an asymptotically valid $1-\alpha$ confidence interval for $\mathbb E[f(\boldsymbol{Y})]$, as $n,N\to\infty$.
\end{theorem}
\begin{proof}{Proof}
    Theorem \ref{unbiased} establishes $\hat{\mu}_{PEMC}$ is unbiased. Moreover, as $\hat{\sigma}^2_{f-g}, \hat{\sigma}^2_{g}$ are consistent estimates of ${\sigma}^2_{f-g}, {\sigma}^2_{g}$, the results follows from Lemma \ref{pemcvariance} and the central limit theorem and Slutsky's theorem.\Halmos
\end{proof}

\subsection{Variance Reduction}
Having analyzed the bias and variance of the PEMC estimator, we now turn to a fundamental practical question: Under a fixed cost budget—whether measured in terms of the number of samples, computational time, or other limited resources—when does PEMC outperform standard MC? Our aim is to understand how PEMC’s variance reduction scales with its cost, thereby identifying regimes where PEMC emerges as more efficient than standard MC. To this end, we first determine the optimal allocation of the sample sizes  $n$  (the number of expensive full simulations) and  $N$  (the number of cheap, feature-only evaluations) within the PEMC framework, given a total resource budget $B$. This investigation provides a guideline for parameter selection and reveals conditions under which PEMC’s gains are maximized.
\begin{lemma}\label{ratio}
Assuming $n,N\in\mathbb R^+$ and $c_g, c_{f-g}\in\mathbb R^+$, the optimal allocation between $n$ and $N$ for PEMC, for any positive budget $B$, satisfies the following relationship:\begin{equation}\label{optimalnandN}
	\frac{n}{N} = \frac{\sigma_{f-g}}{\sigma_g}\cdot\sqrt{\frac{c_{g}}{c_{f-g}}} \quad .
\end{equation}
\end{lemma}

In practice, $ c_{f-g}\gg c_g$ often holds which suggests one should set $N\gg n$ according to Lemma \ref{ratio}.
While Lemma~\ref{ratio} relies on treating  $n,N\in\mathbb R^+$  as continuous—ignoring the integer constraints $n,N \in\mathbb N^+$— this continuous approximation still provides valuable guidance when $n,N\geq 1$. In fact, in all our subsequent experiments, we used $N/n=10$ which are close to the near-optimal ratio suggested by Lemma \ref{ratio} (e.g., in all the use cases in Section \ref{apps}, the optimal ratio $N/n$ is among 5-15). On the other hand, the standard CV that uses a known CV mean corresponds to the case $c_g=0$, and hence $N=\infty$, which reduces to the idealized scenario discussed in Section \ref{pemc overview}. Building on Lemma \ref{ratio}, we can provide an upper bound on the variance reduction PEMC can achieve relative to standard MC under these idealized assumptions.
\begin{theorem}\label{ratioofvariance}
	Assume $c_{f-g}=c_f$. In the same setup as Lemma \ref{ratio}, the variance ratio between PEMC under the optimal allocation and standard MC follows as 
	\begin{align}\label{ration_funda}
		\frac{\mathrm{Var}(\hat{\mu}_{PEMC})}{\mathrm{Var}(\hat{\mu}_{MC})} = \Big(\frac{\sigma_{f-g}}{\sigma_f}+\frac{\sigma_g}{\sigma_f}\sqrt{\frac{c_g}{c_f}}\Big)^2
	\end{align}
\end{theorem}

Based on the result of Theorem \ref{ratioofvariance}, a natural question then arises on how we can gauge $\frac{\sigma_{f-g}^2}{\sigma_f^2}$ or $\frac{\sigma_{g}^2}{\sigma_f^2}$ in practice, and also how these ratios tie to the quality of the predictive model $g$ in PEMC. As we shall see, the extent of variance reduction that PEMC can deliver is linked to how well $g$ is trained. To illustrate this, we first consider an ideal scenario where $g$ is trained to the optimum, i.e., the true regression function.
\begin{theorem}\label{idealg}
    Assume $c_{f-g}=c_f$, $f(\boldsymbol{Y})$ is square-integrable and $g = \mathbb{E}[f(\boldsymbol{Y}) \,|\,  \boldsymbol{X}]$. Define $\rho:=\mathrm{corr}(f, g)$ and $c:=\frac{c_g}{c_{f}}$. Then we have $\rho = \frac{\sigma_g}{\sigma_f}$ and the ratio $\frac{\mathrm{Var}(\hat{\mu}_{PEMC})}{\mathrm{Var}(\hat{\mu}_{MC})} $ in Theorem \ref{ratioofvariance} reduces to
    \begin{align}\label{ratio_number}
		 r(\rho,c):= &\Big(\sqrt{1-\rho^2}+\rho\sqrt{c}\Big)^2.
	\end{align}
\end{theorem}

Theorem~\ref{idealg} says that the variance reduction of PEMC relative to standard MC can be approximated by the function $r(\rho,c)$ in \eqref{ratio_number}, where $\rho = \sigma_g / \sigma_f$ captures the predictive power of $g(\boldsymbol{X})$ for $f(\boldsymbol{Y})$, and $c= c_g / c_f$ quantifies the relative cost of evaluating $g(\boldsymbol{X})$ versus generating full samples for $f(\boldsymbol{Y})$. With this, \eqref{ratio_number} has an intuitive interpretation: The term $\sqrt{1-\rho^2}$ captures the irreducible noise in the CV. This is in the sense that, in conventional CV where $c$ is essentially set as 0, this becomes the only term left for $r(\rho,c)$, which is precisely the well-established optimal standard deviation attained by a CV. On the other hand, the term $\rho\sqrt{c}$ is the cost-driven penalty that depends on both $c$ and $\rho$. 

In Figure \ref{fig1}, we visualize the variance reduction function $r(\rho,c)$ for $\rho,c\in(0,1)$. On the left is a contour plot of $r(\rho,c)$ where the red line traces the curve $r(\rho,c)=1$, corresponding to the ``break-even" boundary where PEMC's variance matches that of MC. Regions below this line, $\{(\rho,r)\in[0,1]^2|r(\rho,c)<1\}$, represent regimes of $\rho$ and $c$ where PEMC achieves variance reduction, i.e., where $\rho$ is sufficiently large and $c$ is sufficiently small. On the right, the graph of $r(\rho,c)$ as a function of $\rho$ when $c=0.001$ provides a clear benchmark: a correlation $\rho=0.5$ yields approximately 22.2\% variance reduction, while $\rho=0.7$ yields about 45.8\%. In our PEMC applications, we design feature $\boldsymbol{X}$ and predictive model $g$ so that, for most $\boldsymbol{\theta}\in\Theta$ during evaluation, we can safely gauge that $c$ falls between $10^{-2}$ to $10^{-3}$ and $\rho$ exceeds 0.5. In these regions, Lemma~$\ref{ratio}$ suggests choosing $N/n$ in the range of 5 to 20 for near-optimal performance. In our experiments, a ratio of $N=10n$ proved effective.
\begin{figure}
     \FIGURE
   {\includegraphics[scale=0.3]{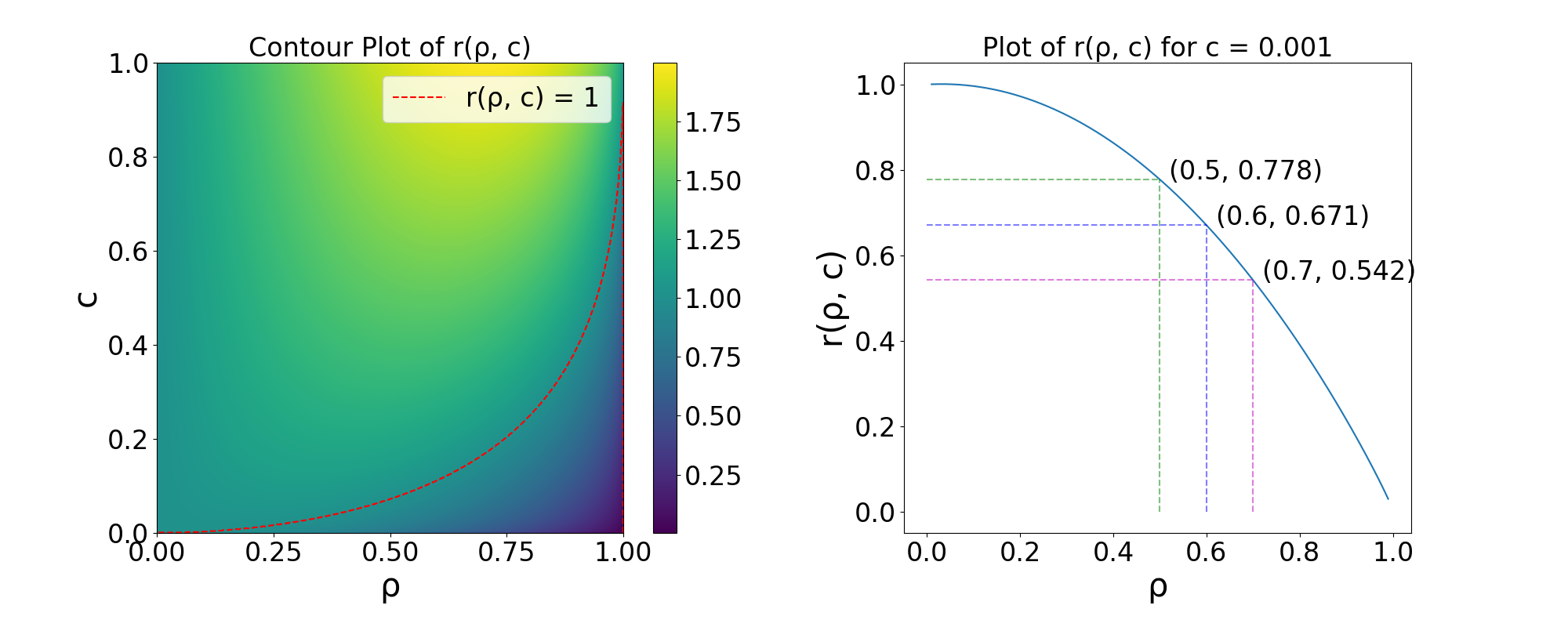}}  
   {Variance Reduction Function $r(\rho,c)$. \label{fig1}}
{Left figure: A contour map of the variance reduction ratio  $r(\rho,c)$  as a function of the correlation $\rho = \sigma_g / \sigma_f$ and relative cost  $c = c_g / c_f$ . The red curve indicates the curve  $r(\rho,c) = 1$ and the regimes where PEMC attains variance reduction over MC lie below the line.  Right figure: A graph of  $r(\rho,c)$ as a function of $\rho$ when $c = 10^{-3}$. }
\end{figure}

\subsection{Incorporating Model Training Errors in Asymptotic Performance Guarantees}\label{LTE}

The statement in Theorem~\ref{idealg} assumes that $g$ represents the true regression function $\mathbb{E}[f|\boldsymbol{X}]$. In practice, however, $g$ is obtained through a learning procedure applied to finite training samples, and thus will only approximate the true conditional expectation. Consequently, any practical implementation of PEMC must account for the approximation and estimation errors inherent in the ML training process, and this error also varies with $\boldsymbol\theta$. 
In the next lemma, we establish the typical asymptotic performance guarantee of our PEMC framework, whose proof is detailed in the Appendix \ref{proofoflemma6}.  

\begin{lemma}\label{learning_theory_summary}
     Under the regularity conditions specified in Appendix \ref{ltproofsection}, for any $\epsilon,\delta > 0$, there exists a sufficiently large sample size $N^0_{\text{train}}$ and a suitably chosen NN class from which we can select and train a predictor ${g}$ on $N_{\text{train}}\geq N^0_{\text{train}}$ samples, such that
     \begin{equation*}
         \frac{\mathrm{Var}(\hat{\mu}_{PEMC})}{\mathrm{Var}(\hat{\mu}_{MC})}  \leq r(\rho,c) +\epsilon
     \end{equation*}
     for a randomly evaluated $\theta\sim \Theta$ with probability at least $1-\delta$.
\end{lemma}

 \subsection{ Control Variate Coefficient}

In this section, we briefly discuss the connection between the PEMC framework and the traditional CV method, in regard to the linear coefficient in front of the CV typically considered in the traditional approach. In \eqref{estimator}, note that we could introduce a free parameter $a$ to create a variant:
\begin{equation}\label{estimator_a}
        \hat{\mu}_{PEMC}(a):=\frac{1}{n}\sum_{i=1}^n(f(\boldsymbol{Y}_i)- ag(\boldsymbol{X}_i))+\frac{1}{N}\sum_{j=1}^N ag(\boldsymbol{\tilde X}_j)
    \end{equation}
The introduction of $a$ is analogous to conventional CV approaches, where the parameter $a$ typically introduces only a negligible bias (when tuned using reused MC samples) but can help minimize variance. In PEMC, this viewpoint may appear somewhat redundant since $ag$ is also a valid prediction model, which suggests $a$ would be chosen implicitly during the training of the prediction model. To understand this, suppose $g$ is fixed. Then $\frac{1}{n}\sum_{i=1}^n -g(\boldsymbol{X}_i)+\frac{1}{N}\sum_{j=1}^N g(\boldsymbol{\tilde X}_j)$ can be viewed as a readily available, zero-mean CV for $\frac{1}{n}\sum_{i=1}^n f(\boldsymbol{Y}_i)$. We can then choose $a$ which minimizes the variance. The optimal $a^*$ then is:
\begin{equation*}
   a^* = \frac{\mathrm{Cov}(f(\boldsymbol{Y}),g(\boldsymbol{X}))}{(n/N+1)\mathrm{Var}(g(\boldsymbol{X}))}.
\end{equation*}
 In the ideal scenario, we have $N\gg n$ and $g(\boldsymbol{X})=\mathbb E[f(\boldsymbol{Y})|\boldsymbol{X}]$, which would imply $\mathrm{Cov}(f(\boldsymbol{Y})-g(\boldsymbol{X}),g(\boldsymbol{X}))=0$ and $\mathrm{Cov}(f(\boldsymbol{Y}),g(\boldsymbol{X})) = \mathrm{Var}(g(\boldsymbol{X}))$. This leads to $a^*=\frac{1}{1+n/N}\approx 1$, the default choice of PEMC, i.e., $a=1$. In other words, by minimizing MSE in the pre-training phase to nearly attain $g(\boldsymbol{X})=\mathbb E[f(\boldsymbol{Y})|\boldsymbol{X}]$, the tuning of the linear coefficient $a$, if we choose to incorporate in PEMC, is implicitly accounted for and set to 1.

\subsection{Unifying Analysis under Both Training and Evaluation Errors}\label{unifiedanalysis}
Finally, we bring together the preceding theoretical results for PEMC by explicitly connecting (i) ML training quality and (ii) the design choice of $\boldsymbol X$ to the break-even comparison for PEMC on evaluations. Lemma \ref{learning_theory_summary} provides an asymptotic, high-level view of the achievable variance reduction as the training sample size grows. Here, we substantiate it by providing a more operational, finite-sample assessment of PEMC's cost-variance trade-off that can be used in an analogous way as Theorem \ref{ratioofvariance}. To this end, note that in principle, once a ML model $g$ is trained, the key ingredients in Theorem \ref{ratioofvariance} - namely $\frac{\sigma_{f-g}}{\sigma_{f}}$, $\frac{\sigma_g}{\sigma_f}$ and cost ratio $\frac{c_g}{c_f}$ are directly estimatable. Consequently, setting aside the one-time training cost, one can check if PEMC will break-even with MC using Theorem \ref{ratioofvariance}. The limitations of Theorem \ref{ratioofvariance}, however, is that it does not by itself directly reveal how the quality of feature design $\boldsymbol X$ and the statistical/optimization accuracy of learning $g$ jointly affects the performance of PEMC. 

To make the above dependence explicit, we combine Theorem \ref{ratioofvariance} with Theorem \ref{idealg} to obtain the following theorem:

\begin{theorem}\label{learningquality}
    Suppose we have trained an ML model $g$ such that for a fixed $\boldsymbol{\theta}$
    \begin{equation*}
        \mathbb{E}_{\boldsymbol X(\boldsymbol{\theta})\sim\text{risk neutral measure}(\boldsymbol{\theta})}
	\biggl((g - g^*)(\boldsymbol{\theta}, \boldsymbol{X}(\boldsymbol{\theta}))\biggr)^2 \leq \epsilon
    \end{equation*}
    where $g^*(\boldsymbol{\theta}, \boldsymbol X(\boldsymbol{\theta})):= \mathbb E[f(\boldsymbol Y(\boldsymbol{\theta}))|\boldsymbol{\theta}, X(\boldsymbol{\theta})]$. Then, 
    \begin{equation}\label{idealgmore}
         \frac{\mathrm{Var}(\hat{\mu}_{PEMC})}{\mathrm{Var}(\hat{\mu}_{MC})}  \leq \Big(\sqrt{1-\rho^2+\frac{\epsilon}{\sigma_f^2}}+(\rho+\frac{\sqrt{\epsilon}}{\sigma_f})\sqrt{c}\Big)^2
     \end{equation}
where $\rho:= \frac{\sigma_{g^*}}{\sigma_f} = corr (g^*(\boldsymbol{\theta}, \boldsymbol X(\boldsymbol{\theta})), f(\boldsymbol Y(\boldsymbol{\theta})))$ and $c:=\frac{c_g}{c_f}$.
\end{theorem}

Theorem \ref{learningquality} gives a clearer picture on when and where to use PEMC, and how exactly the quality of ML model and the design $\boldsymbol X$ directly affect the break-even analysis of PEMC over MC:
\begin{enumerate}
    \item The predictive power of $\boldsymbol X$ dictates $\rho$, i.e. the correlation between $f(\boldsymbol Y)$ and $g^*(\boldsymbol X)$.
    \item The computational advantage of $\boldsymbol X$ dictates the cost ratio $c=\frac{c_g}{c_f}$, i.e., the cost of generating $\boldsymbol X$ versus $\boldsymbol Y$.
    \item The training quality of the ML model $g$ dictates $\epsilon$, i.e., the residual error w.r.t $g^*$.
\end{enumerate}

In particular, in the idealized limit of perfect training, i.e., $\epsilon=0$, \eqref{idealgmore} reduces to the optimal-regressor (i.e., $g=g^*$) case  \eqref{ratio_number} in Theorem \ref{idealg} and one can check $r(\rho,c)$ in Figure \ref{fig1} to see how candidate $\boldsymbol X$ needs to balance both $c$ and $\rho$ to reach break-even. On top of that, imperfect training of the ML model, including the choice of function class and the optimization algorithms, affects the break-even analysis through the error term $\epsilon$ in \eqref{idealgmore}. As an appeal, Theorem \ref{learningquality} is agnostic to the training approach or the magnitude of the error (relative to training size or other algorithmic parameters), as the bound \eqref{idealgmore} inherits the training error $\epsilon$ explicitly. Typically, this $\epsilon$ is of order $\text{(Complexity of model family)}/\sqrt{N_{train}}$, but as mentioned this is not needed as an assumption in Theorem \ref{learningquality}. 

The results in Theorems~\ref{ratioofvariance}, \ref{idealg} and \ref{learningquality} are for fixed $\boldsymbol{\theta}$. To obtain the performance of PEMC for random $\boldsymbol{\theta}$, one first formalize
\[
    \rho(\boldsymbol{\theta}) := corr (g^*(\boldsymbol{\theta}, \boldsymbol X(\boldsymbol{\theta})),f(\boldsymbol Y(\boldsymbol\theta))),
    \quad
    c(\boldsymbol{\theta}) := \frac{c_g(\boldsymbol{\theta})}{c_f(\boldsymbol{\theta})}. 
\]
and use statistical learning theory to obtain $\mathbb E_{\boldsymbol{\theta}\sim \Theta}[\epsilon(\boldsymbol{\theta})] \leq \epsilon$ or high probability bound on $\epsilon(\boldsymbol{\theta})$. Then, by applying Theorem \ref{learningquality}, to bound $\frac{\mathrm{Var}(\hat{\mu}_{PEMC})}{\mathrm{Var}(\hat{\mu}_{MC})}$, in principle one needs to provide bounds for $\rho(\boldsymbol{\theta})$ and $c(\boldsymbol{\theta})$ in the applied region $\Theta$. The following simple example provides a case where the impact of training and model quality can be explicitly illustrated.

\subsubsection{Illustrative Theoretical Example.}
Let us consider a setup where $\theta\sim\mathcal N(0,\sigma^2_{\theta})$, $X|\theta\sim\mathcal N(\theta,\sigma_X^2)$ and $Y=\theta+X+Z$, with $Z$  independent of everything else. The generation of $Z$ is time-consuming and its distribution is only known up to its first two moments, i.e., $\mathbb E[Z] = 0$ and $Var(Z) = \mathbb E[Z^2] = \sigma^2_Z$. For simplicity, let us also assume $c:=\frac{1}{c_Z}$, i.e., we have a unit cost for generating $X$ and the cost of generating $Z$ is $c_z$ times it. Suppose we generate many pairs of $\{(\theta_i, Y_i, X_i)\}_{i\in[N_{train}]}$, and we use linear regression as our ML modeling class to regress $Y$ on $(\theta,X)$. Assume we are fitting a linear model without intercept, i.e., $g(\theta, X)= \beta_1\theta+\beta_2X$. Given the training data, the fitted coefficient $\hat{\beta}_2$ is random but satisfies
\[
    \mathbb E[\hat{\beta}_2]= 1,
    \quad
    Var(\hat{\beta}_2) = \frac{\sigma_Z^2}{(N_{train}-3)\sigma_X^2},
\]
and at evaluation stage, we have from Theorem \ref{ratioofvariance} that
\begin{equation}\label{goodtheoryexample}
\frac{\operatorname{Var}(\mathrm{\hat{\mu}_{PEMC}})}{\operatorname{Var}(\mathrm{\hat{\mu}_{MC}})}
=
\left(
\frac{\sqrt{(1-\hat{\beta}_2)^2\sigma_X^2+\sigma_Z^2}}{\sqrt{\sigma_X^2+\sigma_Z^2}}
+
\frac{|\hat{\beta}_2|\,\sigma_X}{\sqrt{\sigma_X^2+\sigma_Z^2}}\;\frac{1}{\sqrt{c_z}}
\right)^2,
\end{equation}
at the realized $\theta$. This serves as an illustrative example that concisely captures PEMC’s theoretical advantages and summarizes our theoretical results:
\begin{enumerate}
    \item The choice of $X$ determines the predictive power $\rho=\frac{\sigma_X}{\sqrt{\sigma_X^2+\sigma_Z^2}}$ and computational advantage $c=\frac{1}{c_Z}<1$. 
    \item When training reaches the theoretical optimum, i.e., $\hat{\beta}_2=1$, the optimal variance reduction is reached for the given choice of $X$: $(\sqrt{1-\rho^2}+\rho\sqrt{c})^2$.
    \item The training error for $\hat{\beta}_2$ has standard deviation $\frac{\sigma_Z}{\sqrt{N_{train}-3}\sigma_X}$.
\end{enumerate}
Therefore, in the regime where $\frac{\sigma_X^2}{\sigma_Z^2}$, $c_Z$ and $N_{train}$ (i.e., the predictive power of $X$, the cost of generating $Z$ and the supply of training data) are sufficiently large, the training data will, with high probability $1-O(\frac{1}{\sqrt{N_{train}}})$, yield a PEMC that has variance reduction over MC for all possible ${\theta}$ in the evaluation stages, i.e., \eqref{goodtheoryexample} has value less than 1.

\subsubsection{Further Discussion on the Role/Cost of Training}
When discussing PEMC performance, we only consider evaluation-stage efficiency, because PEMC is designed to reduce latency and variance per evaluation query. Accordingly, we treat the ML predictor $g$ as a fixed, pre-trained asset and regard the training cost $C_{\text{train}}$ as a sunk, offline expense. This is much like the up-front, problem-specific cost (e.g., deriving the mean in closed-form) one must pay before deploying classical control variates, which is not charged against each simulation run.

Under this viewpoint, the cost $C_{\text{train}}$ is incurred offline, ahead of time and likely on different hardware (e.g., GPU cluster) and does not affect the latency or the per-query cost. If, however, the user’s cost objective is end-to-end and would like to include a one-time model development cost, then $C_{\text{train}}$ can in principle be explicitly folded into the comparison, but this requires a user- and application-specific accounting: one must specify a cost exchange ``rate'' between online and offline compute (which could differ not only in hardware or timing but also in monetary price or business implications) and the deployment horizon, i.e., how many queries PEMC will ultimately serve in evaluation.  For example, if one trains PEMC once and evaluate online indefinitely, then the offline term is amortized away and the end-to-end objective reduces to exactly the evaluation-stage analysis we develop. 

More fundamentally, folding in a one-time training cost is not how PEMC is designed. In practice, a more natural and realistic way to deploy PEMC is online — the predictor is trained and periodically refreshed as new simulation data arrive, reusing the samples generated during evaluation to sharpen $g$ over time while making efficient use of essentially all available data:

\begin{itemize}
    \item \textbf{Train-then-evaluate with lagged updates:} At time $t$, evaluate using a model $g_t$ trained on data available up to time $t-1$; after the evaluation is completed, the newly generated evaluation data can be added to the training pool to produce $g_{t+1}$. This preserves the required independence for each evaluation while allowing streaming evaluation data to be incorporated going forward.
    \item \textbf{Cross-validation (for training quality, not unbiasedness):} Standard ML practices such as cross-validation/holdout validation, early stopping, and regularization can be used to tune architectures/hyperparameters and curb overfitting; after selecting the model class, one can retrain on the full accumulated dataset. This practice can be embedded directly into the construction of $g_t$, and will not affect unbiasedness.

    \item \textbf{Optional sample splitting:} If one insists on updating $g_t$ using the same batch of data that is also used for evaluation at time $t$, then a simple holdout split can be used to preserve independence.
\end{itemize}

\section{Applications}\label{apps}
We apply PEMC to three real-world examples: Pricing two  exotic options, namely variance swaps under stochastic local volatility models (Section \ref{sec:var swap}) and swaptions under the Heath-Jarrow-Morton model (Section \ref{sec:swaption}), and an ambulance diversion policy evaluation problem (Section \ref{sec:EMS}). In Appendix \ref{xvastuff}, we make several additional observations on implementation and evaluation metrics.

\subsection{Variance Swaps in Stochastic Local Volatility Models}\label{sec:var swap}

Variance swaps are financial derivatives that enable investors to trade future realized volatility against current implied volatility. 
Unlike traditional options, which derive their value from the underlying asset's price, variance swaps are based on the variance of the asset's returns over a specified period \citep{BayerDermanGourlayScott1999}. This unique structure allows for pure exposure to volatility, making variance swaps valuable tools for risk management and speculative strategies.  

In this subsection, we focus on using PEMC to price variance swaps under a stochastic local volatility (SLV) model \citep{ren2007calibrating}. Closed-form formulas exist for variance swaps in stochastic volatility models such as Heston.  
However, there are no explicit expressions under SLV \citep{tataru2010stochastic}, even though this model provides greater flexibility through a data-intensive, non-parametric specification of the volatility surface \citep{gatheral2006volatility}. SLV models emerged as a hybrid approach that combines the market-implied local volatility surface with stochastic volatility dynamics \citep{guyon2014nonlinear}, providing practitioners with greater calibration flexibility and more accurate price reproduction across strike-maturity ranges \citep{ren2007calibrating}. Unlike the Heston model's parsimonious 5 parameter set, SLV requires handling a full volatility surface discretized on a dense 2D grid, effectively making the parameter space $\boldsymbol{\Theta}$ for PEMC high-dimensional. To efficiently process this grid-structured volatility data, we adopt a Convolutional Neural Network (CNN) architecture \citep{lecun1998gradient} for PEMC, effectively feeding the local vol surface as an image into predictive modeling. The CNNs are particularly well-suited for this task as they naturally exploit the spatial relationships in the volatility surface. 

Under the SLV model, we adopt the following SDEs for simulating the asset price $\{S_t\}_t$:
\begin{align}\label{localvolsde}
    \mathrm{d} S_t &= r S_t \mathrm{d} t + \sigma(S_t,t) e^{\nu_t} S_t \mathrm{d} W_t^S, \\
    \mathrm{d} \nu_t & = \kappa (\eta_t - \nu_t) \mathrm{d} t + \delta \mathrm{d} W_t^{\nu},
\end{align}
with $\langle \mathrm{d} W_t^S, \mathrm{d} W_t^{\nu}\rangle = \rho dt$ and  $\eta_t := -\frac{\delta^2}{2\kappa}(1+e^{-2\kappa t})$ with  $\kappa \in \mathbb{R}^+$, and $\delta \in \mathbb{R}$. Here $\sigma(\cdot,\cdot):\mathbb R\times\mathbb R^+\to \mathbb R^+$ is 2D function representing the (interpolated) local volatility surface and $\exp(\nu_t)$ is a stochastic multiplier with $\exp(\nu_0)=1$ and $\mathbb{E}[e^{2\nu_t}]=1$. In practice, local volatility surfaces are calibrated to and stored as discrete two-dimensional grids indexed by asset prices (spot) and time. During simulation, these discrete values are interpolated as needed to obtain volatility at arbitrary price-time points \citep{coleman2001interpolation}. For our PEMC implementation, we treat this discrete grid as part of our input parameter $\boldsymbol{\theta}_{\text{model}}$, reflecting how the market-calibrated surfaces would be used in practice. While local volatility models \citep{dupire1994pricing, derman1994riding} and their calibration to market data constitute an extensive research area in their own right, our focus here is on the PEMC implementation. Thus, we assume a calibrated local volatility has been given for evaluation, regardless of the method used to obtain it. Following practical conventions, we store $\sigma(\cdot,\cdot)$ on a $|\mathcal{S}|\times |\mathcal{T}|$ grid, where $\mathcal{S}$ contains $|\mathcal{S}|$ equally-spaced price points in $[S_{\text{surface}}^{\min}, S_{\text{surface}}^{\max}]$ and $\mathcal{T}$ contains $|\mathcal{T}|$ equally-spaced time points in $[t_{\text{surface}}^{\min}, t_{\text{surface}}^{\max}]$. At each point on the grid, we store the value of local volatility according to 
\cite{carmona2007hjm}:
\begin{equation}\label{eqn:local-vol}
\sigma_{base}^2(x,t) = \frac{\sum_{i=0}^2 p_i\tau_ie^{-x^2/(2t\tau_i^2)-t\tau^2_i/8}}{\sum_{i=0}^2(p_i/\tau_i)e^{-x^2/(2t\tau_i^2)-t\tau_i^2/8}}, \quad \text{with}~~~p_0:=1-p_1-p_2,~~x:=\log(S_t/S_0),
\end{equation}
for $p_0=0.3, p_1=0.5, p_2 =0.2, \tau_0=0.4, \tau_1=0.3,\tau_2=0.6$ as in Figure 2 in \cite{carmona2007hjm}. While this analytical form \eqref{eqn:local-vol} is used in our data-generating process, it is important to note that PEMC treats the surface as any market-calibrated volatility surface - accessing it only through its discrete grid values. When sampling $\boldsymbol{\theta}\sim\boldsymbol\Theta$ to produce $\sigma$ grid, we add a $\mathcal N(0,\xi^2)$ noise independently to all the $|\mathcal{S}|\times |\mathcal{T}|$ points in the grid, on top of their baseline value $\sigma_{base}$ \eqref{eqn:local-vol}. A path is then generated by Euler's scheme:
\begin{align*}
\nu_{t+\Delta t} & \leftarrow \nu_t + \kappa(\eta_t - \nu_t)\Delta t + \delta \Delta W^\nu_t, \\
S_{t+\Delta t} & \leftarrow S_t \exp\Big( \big(r - \tfrac{1}{2} \tilde{\nu}_{t+\Delta t}^2\big)\Delta t + \tilde{\nu}_{t+\Delta t}\Delta W^S_t  \Big), \quad \text{where}~~\tilde{\nu}_{t+\Delta t} = \widehat{\sigma}(S_{t}, t)e^{\nu_{t+\Delta t}},
\end{align*}
with $(\Delta W^S_t, \Delta W^\nu_t)\stackrel{\text{i.i.d}}{\sim}\sqrt{\Delta t}\cdot \mathcal{N}(\mathbf{0},[\begin{smallmatrix} 1 & \rho \\ \rho & 1 \end{smallmatrix}])$, and $\widehat{\sigma}(S_{t}, t)$ obtained by interpolations of the grid $\sigma$. The input $(\boldsymbol{\theta},\boldsymbol{X})$ can be represented as 
\begin{equation*}
    \text{feature}_i=\{ \underbrace{\{\sigma^2(s,t)\}_{s\in\mathcal S,t\in\mathcal T}}_{\text{surface info }}, \underbrace{r,\delta, \kappa, \rho, \mu}_{\boldsymbol{\theta}_{\text{model}}}, \underbrace{S_{0},\nu_{0}}_{\boldsymbol{\theta}_{\text{simulation}}}, \underbrace{K}_{\boldsymbol{\theta}_{\text{payoff}}}, \underbrace{(W^S_T,W^\nu_T)}_{\mathbf{X}(\boldsymbol{\theta})}\}.
\end{equation*}
The parameter space $\boldsymbol\Theta$, as well as the evaluation $\boldsymbol{\theta}$, is summarized in Table \ref{tab:param-slv-variance-swap}.
\begin{table}[H]
\centering
\caption{Parameter Setup in SLV Model}\label{tab:param-slv-variance-swap}
\scriptsize
\resizebox{0.95\textwidth}{!}{%
\begin{tabular}{c|c|c|c|c|c|c|c}
\hline
mode  & $\xi$ & $S_0$ & $r$ & $\kappa$ & $\lambda$ & $\rho$ & $T$, $\Delta t$ \\ \hline
$N_{\text{train}}=3,000,000$  & $0.02$ &  $[50,150]$ & $0.02$ & $[1.5,4.5]$ & $[0.1,1.0]$ & $[-0.9,-0.2]$ & \multirow{2}{*}{$1$, $1/252$}\\
evaluation & $0$ & $100$ & $0.02$ & $3.0$ & $0.5$ & $-0.5$   \\
\hline
\end{tabular}
}
\end{table}

To handle the 2D grid of high-dimensional volatility surface data, we design a two-branch NN architecture, which is illustrated in Figure \ref{fig:SLV}. The first branch processes the discretized volatility surface $\sigma^2(x,t)$ using a CNN architecture inspired by VGG \citep{simonyan2014very}, which has become a standard choice for image processing tasks and is readily available in modern deep learning packages like PyTorch \citep{paszke2019pytorch}. This branch then consists of two 2D convolutional layers interspersed with ReLU activations, followed by a MaxPool2d operation for dimensionality reduction. The surface features are then flattened through a fully connected layer to produce an embedding. The second branch handles the remaining model parameters in $(\boldsymbol{\theta}, \boldsymbol X)$ through a series of fully connected layers with dropout regularization, batch normalization, and ReLU activations, ultimately producing another embedding. Finally, the two separate embeddings are then fed into a ``Synthesizer" module, which combines the information through additional fully connected layers with dropout and ReLU activations to produce the final prediction. This architecture choice is motivated by the proven effectiveness of CNNs in handling grid-structured data \citep{lecun1998gradient, he2016deep}, and particularly the VGG architecture's success in extracting hierarchical features while maintaining relative simplicity \citep{simonyan2014very}. The hyper-parameter setup for the NN architecture is summarized in Table \ref{tab:nn_params} of Appendix \ref{Experiment_tab_graph}.

\begin{figure}[h]
    \centering
    \caption{Neural Network architecture for the SLV Model.}
    \includegraphics[width=0.8\textwidth]{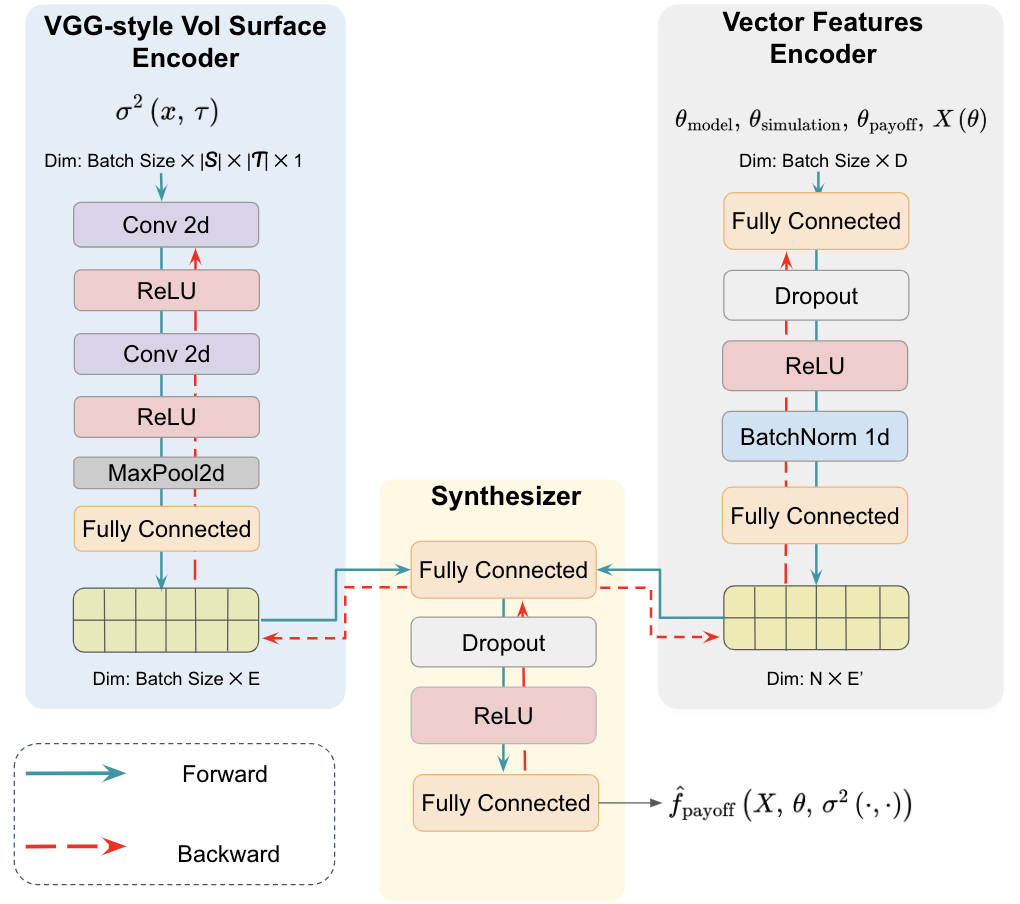}
    \label{fig:SLV}
\end{figure}

Finally, we evaluate PEMC's performance in the SLV setting across sample sizes $n\in\{1000,2000,4000,6000,8000,10000,20000\}$ with $N=10n$, benchmarking against a ground truth computed from $5\times 10^{7}$ MC samples. The results, based on 200 independent experiments and shown in Table~\ref{tab:variance-sawp-result-slv}, reveal that PEMC's effectiveness persists even in this more complex setting. Despite the added complexity of handling high-dimensional volatility surfaces, PEMC achieves a 27-43\% reduction in RMSE compared to standard MC.
\begin{table}[!h]
\centering
\caption{RMSE from 200 Experiments for Pricing Variance Swap under the SLV Model}\label{tab:variance-sawp-result-slv}
\small
\begin{tabular}{l|ccccccc}
\hline
 Method &  $n=1000$ & $n=2000$ & $n=4000$ & $n=6000$ & $n=8000$ & $n=10000$ & $n=20000$ \\ \hline
 Standard MC & 0.0206 &	0.0145&	0.0101&	0.0075&	0.0064&	0.0065&	0.0047 \\
 PEMC  & 0.0130&	0.0088&	0.0061&	0.0055	&0.0040&	0.0040&	0.0027\\
 \hline
\end{tabular}
\end{table}

\subsection{Swaptions in HJM Models}\label{sec:swaption}
Interest rate derivatives play a crucial role in financial markets, with swaptions being particularly important instruments for managing interest rate risk \citep{brigo2006interest}. A swaption gives its holder the right to enter into an interest rate swap at a future date, providing flexibility in hedging future interest rate exposures \citep{hull2018options}. The pricing of these instruments, however, presents significant computational challenges due to the high-dimensional nature of interest rate models.
In this subsection, we demonstrate PEMC's application to swaption pricing under the Heath-Jarrow-Morton (HJM) framework \citep{heath1992bond}.\footnote{For completeness, we provide a short review of the HJM model in Appendix \ref{hjmappendix}.}

Just as with the local volatility surface case, in practice $\sigma(t,T)$ cannot be predefined parametrically and must be calibrated from market data of caps, floors, and swaptions, yielding a discrete grid of values. However, for demonstration purposes, we employ a classical exponential decay specification as our baseline model inspired from \cite{glasserman2013monte}:
\begin{equation}\label{eqn:hjm-vol}
\sigma_{base}(t,T) = \sigma_0\exp(-\alpha_\sigma(T-t))
\end{equation}
where $\sigma_0$ and $\alpha_\sigma$ parameterize the baseline
volatility structure. Similarly, for a baseline initial forward curve, we use:
\begin{equation}\label{eqn:hjm-forward}
f_{base}(0,T) = f_0 + c_f(1-\exp(-\alpha_f T)).
\end{equation}
where $f_0$, $c_f$, and $\alpha_f$ parameterize the baseline initial forward curve. While this analytical form serves as our data-generating process, PEMC accesses it only through its discrete grid values with added noise. This approach mirrors our treatment of the local volatility surface in the previous section, where we used the parametric form in \cite{carmona2007hjm} solely as a realistic baseline for generating training data. 

Indeed, while HJM and these analytical forms \eqref{eqn:hjm-vol}-\eqref{eqn:hjm-forward} are formulated in continuous time, in practice we need to implement numerical discretization. Following the scheme in \cite{glasserman2013monte}, we discretize the
simulation-time axis over $[0,t'_0]$ and the maturity axis over
$[0,T^*]$. Specifically, let
$\mathcal T_t=\{t_0,t_1,\ldots,t_{N_T}\}$ denote the simulation-time
grid and $\mathcal T=\{T_0,T_1,\ldots,T_{N_M}\}$ denote the maturity
grid. We use the same spacing $\Delta t$ on both axes, although the two
grids generally have different lengths.
For each training instance, we first sample
$(\sigma_0,\alpha_\sigma,f_0,c_f,\alpha_f)$ and then construct the
grids
$\{\sigma(t_k,T_j)\}_{t_k\leq T_j}$ and
$\{f(0,T_j)\}_{T_j\in\mathcal T}$
from \eqref{eqn:hjm-vol} and \eqref{eqn:hjm-forward}, respectively.
We add independent Gaussian perturbations $\mathcal N\left(
0,\left(\frac{\sigma_0}{2(T_j+5)}\right)^2
\right)$ and $\mathcal N\left(
0,\left(\frac{1}{100(T_j+5)}\right)^2
\right)$
to the corresponding baseline values.
The forward-rate paths are then generated according to
\[
f(t_k,T_j)
=
f(t_{k-1},T_j)
+
\mu(t_{k-1},T_j)(t_k-t_{k-1})
+
\sigma(t_{k-1},T_j)\Delta W_k,
\]
where
\[
\Delta W_k:=W(t_k)-W(t_{k-1})
\sim\mathcal N(0,t_k-t_{k-1})
\]
is the scalar Brownian increment of the one-factor HJM model, and
$\mu(t_{k-1},T_j)$ is the discretized drift determined by
\eqref{noarb}. The same increment $\Delta W_k$ drives all forward
maturities $T_j$, with maturity-dependent loading
$\sigma(t_{k-1},T_j)$. See \cite{glasserman2013monte} for the complete
derivation of this standard discretization.

Across training instances, we sample
\((\sigma_0,\alpha_\sigma,f_0,c_f,\alpha_f)\) uniformly from the
ranges reported in Table~\ref{tab:param-hjm-swaption}.
The parameters \((\sigma_0,\alpha_\sigma)\) specify the baseline
volatility structure in \eqref{eqn:hjm-vol}, while
\((f_0,c_f,\alpha_f)\) specify the baseline initial forward curve in
\eqref{eqn:hjm-forward}. Evaluating these specifications on the simulation-time and maturity
grids and adding the perturbations described above yields the
discretized volatility structure and initial forward curve used in the
HJM simulation. For each generated initial forward curve, the contractual fixed rate
\(R\) is computed deterministically rather than sampled independently.
Let $T^*:=t'_{n_p}$ denote the terminal date of the underlying swap, and let
$B(0,T)=\exp\left(-\int_0^T f(0,u)\,du\right)$ denote the initial discount
curve implied by $f(0,\cdot)$ (see Appendix \ref{hjmappendix}). We set
\[
R
=
\left(
\frac{B(0,T^*)}{B(0,t'_0)}
\right)^{-1/(T^*-t'_0)}
-1
=
\exp\left(
\frac{1}{T^*-t'_0}
\int_{t'_0}^{T^*}f(0,u)\,du
\right)-1.
\]
Thus, \(R\) varies across pricing instances but is held fixed across
all simulated HJM paths within a given instance. The contract parameters \(C\), \(t'_0\), \(\Delta t'\), and \(n_p\),
as well as the simulation parameters \(\Delta t\) and \(T^*\), are fixed
at the values reported in Table~\ref{tab:param-hjm-swaption}. For
evaluation, we use one randomly sampled volatility structure and
initial forward curve, shared as the ground-truth.

\begin{table}[h]
    \centering
    \caption{Parameter Setup in the HJM Swaption Experiment}
    \label{tab:param-hjm-swaption}
    \resizebox{1\textwidth}{!}{%
    \begin{tabular}{l|ccccc|ccccccc}
        \hline
        \textbf{Mode}
        & $\sigma_0$
        & $\alpha_\sigma$
        & $f_0$
        & $c_f$
        & $\alpha_f$
        & $R$
        & $C$
        & $t'_0$
        & $\Delta t'$
        & $n_p$
        & $\Delta t$
        & $T^*$ \\
        \hline
        $N_{\text{train}}=3e6$
        & $[0.01, 0.03]$
        & $[0.001, 0.9]$
        & $[0.01, 0.03]$
        & $[0.01, 0.05]$
        & $[0.001, 0.9]$
        & $(\frac{B(0,T^*)}{B(0,t'_0)})^{-1/(T^*-t'_0)}-1$
        & $100$
        & $5$
        & $1$
        & $20$
        & $1/52$
        & $25$ \\
        \hline
    \end{tabular}%
    }
\end{table}

For the $i$-th simulated path, let
$\Delta W_\ell^{(i)}$ denote the Brownian increment over the
$\ell$th simulation interval. To construct the low-dimensional PEMC
feature, we partition the increments over $[0,t'_0]$ into three
consecutive blocks $I_1,I_2,I_3$ and define
\[
\boldsymbol X_i
=
\left(
\sum_{\ell\in I_1}\Delta W_\ell^{(i)},
\sum_{\ell\in I_2}\Delta W_\ell^{(i)},
\sum_{\ell\in I_3}\Delta W_\ell^{(i)}
\right).
\]
This three-block representation retains coarse temporal information
about the Brownian path while remaining inexpensive to simulate.
The three blocks are a feature-engineering choice and do not represent
three HJM factors.

As before, the input to the HJM prediction model is
\[
\text{feature}_i
=
\left\{
\underbrace{
\{\sigma(t_k,T_j)\}_{
    t_k\in\mathcal T_t,\,
    T_j\in\mathcal T,\,
    t_k\leq T_j
},
\quad
\{f(0,T_j)\}_{T_j\in\mathcal T}
}_{\boldsymbol{\theta}_{\text{model}}},
\underbrace{
T^*,\Delta t
}_{\boldsymbol{\theta}_{\text{simulation}}},
\underbrace{
C,R,t'_0,\Delta t'
}_{\boldsymbol{\theta}_{\text{payoff}}},
\underbrace{
\boldsymbol X_i
}_{\text{Brownian features}}
\right\}.
\]
The volatility structure and initial forward curve together constitute
the HJM model input $\boldsymbol{\theta}_{\text{model}}$.

To handle both the two-dimensional grid of volatility structure $\sigma(t,T)$ and the one dimensional grid of initial forward curve $f(0,T)$, we design a three-branch NN architecture, illustrated in Figure \ref{fig:swaption_nn}. The first branch, labeled ``2D Function Encoder", processes the 2D volatility structure grid using a CNN architecture with two 2D convolutional layers, each followed by batch normalization. The branch concludes with an average pooling operation and produces an embedding. The second branch processes the initial forward curve $f(0,T)$ grid through a ``1D Function Encoder" utilizing 1D convolutional layers - a natural choice for sequential data \citep{lecun1998gradient} - followed by batch normalization and average pooling to produce another embedding. The third branch processes the nine-dimensional vector input consisting of the simulation inputs $(T^*,\Delta t)$, the contract inputs $(C,R,t'_0,\Delta t')$, and the three Brownian features $\boldsymbol X_i$ through fully connected layers with batch normalization.

Finally, these three separate embeddings are then fed into a ``Synthesizer" module that combines the information through multiple fully connected layers with batch normalization, ultimately producing the final prediction. This architecture effectively leverages both the spatial structure of the volatility surface through 2D CNNs \citep{simonyan2014very}, the sequential structure of the forward curve through 1D CNNs \citep{oord2016wavenet}, and the simulation, contract, and Brownian-vector inputs through standard deep learning techniques \citep{he2016deep}. The hyper-parameter setup for the NN architecture is detailed in Table \ref{tab:hyper-param-swaption} of Appendix \ref{Experiment_tab_graph}.

\begin{figure}[htbp!]
    \centering
     \caption{Neural network architecture for modeling Swaption payoff.}
    \includegraphics[width=0.8\textwidth]{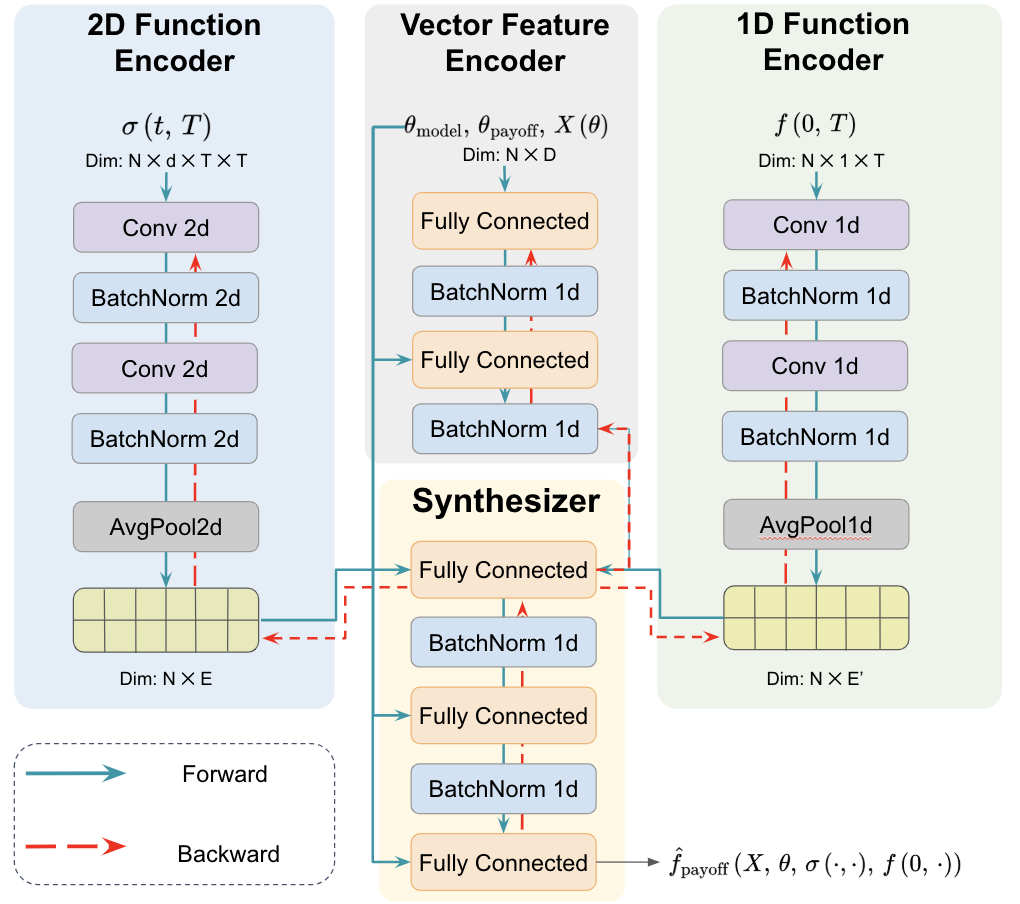}
    \label{fig:swaption_nn}
\end{figure}

Finally, following the same evaluation methodology, we assess PEMC's performance in the HJM swaption pricing setting across sample sizes $n\in\{1000,3000, 5000,7000,9000,11000\}$ with $N=10n$, using $5\times 10^{7}$ MC samples to establish the ground truth value. As displayed in Table \ref{tab_hjm_mse_comparison} and the associated boxplots in Figure \ref{fig:swaption_results}, the outcomes from 200 independent experiments confirm that PEMC remains highly effective within the context of interest rate derivatives. Despite the added complexity of managing both volatility structures and forward rate curves, PEMC achieves a 45-50\% reduction in RMSE to standard MC methods. This performance aligns with the variance reduction levels observed in our earlier examples. 
\begin{table}[h]
    \centering
    \caption{Root Mean Squared Error from 200 Experiments}
    \begin{tabular}{lcccccc}
        \hline
        \textbf{Method} & $n=1000$ & $n=3000$ & $n=5000$ & $n=7000$ & $n=9000$ & $n=11000$ \\
        \hline
        Standard MC & 0.0096 & 0.0062 & 0.0048 & 0.0039 & 0.0035 & 0.0029 \\
        PEMC & 0.0055 & 0.0028 & 0.0024 & 0.0019 & 0.0018 & 0.0015 \\
        \hline
    \end{tabular}
    \label{tab_hjm_mse_comparison}
\end{table}
\begin{figure}[h]
    \centering
    \caption{Boxplots for HJM Experiments.}\label{fig:swaption_results}
    \includegraphics[width=0.8\textwidth]{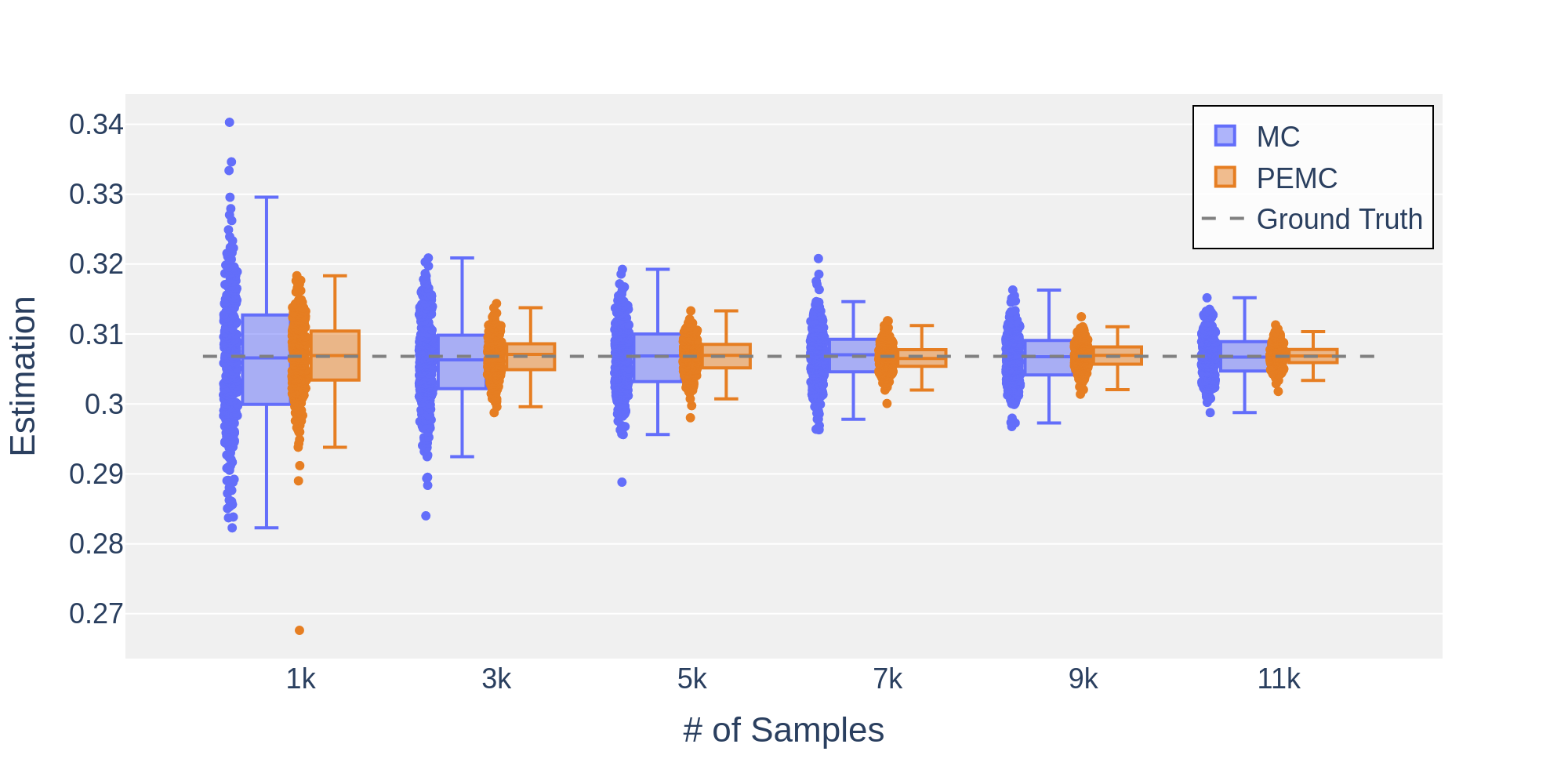}
\end{figure}


\subsection{Ambulance Diversion Policy Evaluation}\label{sec:EMS}
During the early COVID-19 surge, New York City's emergency medical service (EMS) system faced a drastic spatial shift in emergency calls, pushing certain hospitals to crisis-level overload \citep{henrylamwsc_ambul,dolan2022hospital}. In response, a load-balancing rule was devised to divert non-critical ambulance patients away from the closest hospital if that facility nears capacity, balancing travel time against queue length through an optimization-based assignment approach \citep{dolan2022hospital}. A stochastic simulation based on historical EMS data \citep{henrylamwsc_ambul} was used to evaluate how load-balancing reduces peak occupancy and hospital congestion—showcasing how simulation can improve patient outcomes in EMS networks. 

The simulation workflow is a nested simulation system. We focus here on a representative subproblem—the two-hospital ambulance diversion—to illustrate PEMC's advantage. The discrete-event simulation models two emergency departments, ED-1 and ED-2, each with physicians, patients, and arrival streams whose rates $\lambda_{h,d}$ vary by hour $h$ and day $d$. In particular: 1) Patient Arrivals:
We simulate patient inflow to each hospital ($\text{A}$ or $\text{B}$) according to a nonhomogeneous Poisson process with hourly rates $\lambda_{h,d}$. Specifically,
$$\text{Arrivals}(t) \sim \text{NHPP}\bigl(\lambda(t)\bigr),$$
where $\lambda(t)$ is determined by a piecewise schedule over each hour $h$ of day $d$. Many of these arrivals (e.g., 25\%) come by ambulance, with the remainder as walk-ins. 2) Triage Levels and Service Times: 
Each arriving patient is assigned a triage level $\ell \in \{1,2,3,4,5\}$ via a (possibly crisis-adjusted) multinomial distribution:
$$\Pr[L = \ell] = p_{\ell}\text{ or } p_{\ell}^\text{(crisis)}.$$
Service times are exponential with rate $\mu_{\ell}$, depending on the triage level $\ell$. More critical levels (e.g., $\ell=1$) have faster service rates $\mu_1$. 3) Priority Queueing and Threshold Diversion:
$\text{Doctors}$ are modeled as resources, each available over a shift $[\text{start}, \text{end}]$. Patients join a priority queue ordered by triage, preempting lower-priority service. A threshold $\tau$ controls ambulance diversion: if the queue $Q$ in one hospital exceeds $\tau$, new ambulance arrivals divert to the other hospital, incurring additional travel time $\Delta$.
Mathematically,
$$\text{divert}(\tau) = \begin{cases}
1, &\text{if } Q > \tau,\\
0, &\text{otherwise}.
\end{cases}$$
4) Mortality Function: 
Each patient $i$ has a time-varying risk of death. The code uses two complementary representations. The first is direct mortality probability:
$$
\text{mortality\_func}(x,d,a,t,B,\nu)
=
A_t
+
\frac{K_t-A_t}{
\Bigl[1 + (3t)\exp\bigl(-(B+5-t)x + (-2 + 2.5a)t\bigr)\Bigr]^{1/(\nu+0.25t)}
},
$$
where $x$ is the wait time, $d$ indicates diversion, $a$ is whether the patient obtained a doctor, $t$ is triage level, and $(B,\nu)$ control shape and shift. Constants $\{A_t, K_t\}$ reflect baseline and max risk by triage. The second is inverse mortality: $\text{inverse\_mortality\_func}(u,a,t,B,\nu)$ solves for the time at which a random uniform $u$ indicates death. Thus each patient obtains a $\text{death\_time}$ (a form of Weibull-like distribution).
A patient dies if its wait time $W$ exceeds this random $\text{death\_time}$.

Finally, we simulate day-by-day, in 4-hour shifts, updating doctors' availability and arrivals. Diversion is triggered once a queue surpasses $\tau$. We aim to minimize the expected total mortality over $d$ days:
$$
\mathbb{E}\Bigl[\sum_{\text{patients } i}
\mathbf{1}\bigl\{\text{wait}_i > \text{death\_time}_i\bigr\}\Bigr],$$
subject to feasible travel times and doctor schedules.

To validate the proposed load-balancing rule, we develop a discrete-event simulation reflecting real-time dynamics of New York City's EMS system. We optimize $\tau$ to minimize total mortality, measured by the fraction of patients whose wait time surpasses their Weibull survival limit. Experiments show that, particularly in crisis mode, adaptive thresholding significantly reduces overloading and mortality compared to the naive ``always-nearest-hospital" policy. A threshold $\tau$ controls ambulance diversion: if the queue $Q$ in one hospital exceeds $\tau$, new ambulance arrivals divert to the other hospital, incurring additional travel time $\Delta$ and so on. The evaluation needed is how weekly mortality counts depends on the diversion threshold $\tau$. By varying $\tau$ in the simulation, we obtain mortality counts (or rates) as a function of policy choices, enabling downstream decisions. For PEMC predictor we employ a random forest estimator, an ensemble model well-suited for healthcare \citep{Chang2016Triage,pirracchio2015mortality}. We gather a training size $10^5$, and sample $\boldsymbol{\theta}$ from a $\boldsymbol{\Theta}$ calibrated from EMS data. The $\boldsymbol{X}$ are readily simulated components for the nested MC—drawing on Weibull, Exponential, Gamma, and Poisson distributions with $\dim (\boldsymbol{X})=12$ -and is also parallelizable. The results are summarized Figure \ref{fig:mortality_ems} and Table \ref{tab:example_ems}. Further details are left in Appendix \ref{amb_div}.

\begin{figure}[h]
    \centering
    \begin{minipage}{0.45\textwidth}
        \centering
        \caption{Diversion policy evaluation}\label{fig:mortality_ems}
        \includegraphics[width=\linewidth]{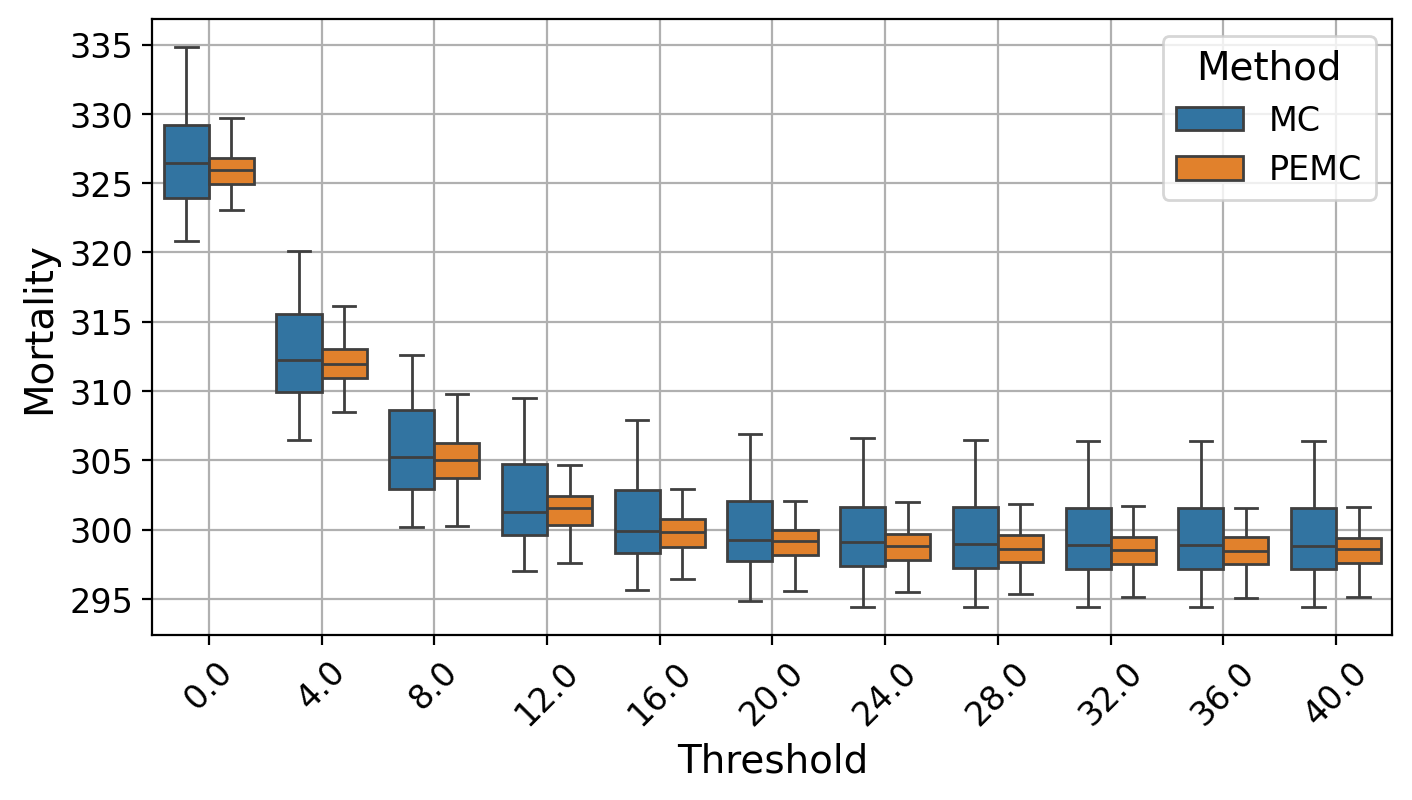}
    \end{minipage}
    \hfill
    \begin{minipage}{0.45\textwidth}
        \centering
        \captionof{table}{MSE and MAE of mortality evaluation.}
        \label{tab:example_ems}
        \vspace{2pt}
        \begin{tabular}{lcccc}
            \hline
            \textbf{Threshold} & \multicolumn{2}{c}{\textbf{MSE}} & \multicolumn{2}{c}{\textbf{MAE}} \\
             & MC & PEMC & MC & PEMC \\
            \hline
            $\tau=0$  & 10.974 & 4.142 & 2.765 & 1.587 \\
            $\tau=20$ & 9.535  & 3.266 & 2.497 & 1.416 \\
            $\tau=40$ & 9.174  & 3.122 & 2.428 & 1.376 \\
            \hline
        \end{tabular}
    \end{minipage}
\end{figure}

\section{Conclusion and Discussion}\label{conclude}

In this paper, we introduced PEMC, a general framework that uses ML–based predictors as flexible CVs. By integrating inexpensive, highly parallelizable simulations as features, PEMC preserves the unbiasedness and error quantification of classical MC while delivering substantial variance and runtime reductions in settings where traditional variance reduction techniques are infeasible. A key idea to realize this gain is a scheme-wide view on the total cost-aware variance reduction for the CV method, as opposed to per-replication variance reduction in the traditional view. This allows us to relax the traditional requirement of CV mean knowledge, thereby opening up the possibility of incorporating a wide range of ML models as valid CV candidates that help substantially reduce variance. We demonstrate PEMC's promising capability in an array of examples in production-grade option pricing and EMS systems.


 Our study paves the way for several promising future research directions. First, meta-learning approaches could enable PEMC to automatically learn effective CVs across related tasks, reducing manual tuning burdens \citep{finn2017model}. Second, recent advances in generative modeling—such as consistency models and diffusion-based methods \citep{ho2020denoising,song2021scorebased}—suggest the possibility of generating advanced proposal distributions or learned variates, though integration into finance requires new theoretical developments. Third, robust regression techniques from causal inference and distribution-shift literature could enhance PEMC’s stability under model mis-specification or nonstationary data \citep{li2020robust,bodnar2022use,ellickson2023estimating,kang2007demystifying,hernan2020causal,angelopoulos2023prediction,quionero2009dataset,bickel2009discriminative,sugiyama2009covariate}. 
 Fourth, though in this paper we only consider PEMC as an enhancement to standard MC, it is a flexible framework that potentially applies to more efficient MC techniques such as Quasi-Monte Carlo methods \citep{caflisch1998monte, aistleitner2014uniform} and their variants 
 to achieve an even higher level of variance reduction.

\arxivonly{Beyond direct expectation estimation, a more speculative extension is to use PEMC as an inner expected-cost estimator within outer stochastic-control or game algorithms. Candidate testbeds include debt-management models with bankruptcy risk and Hamilton--Jacobi formulations \citep{yilun2018optimalopenloopdebt,yilun2018vanishinghjm}, finite-state stochastic Stackelberg games \citep{yilun2020selffeedbackstackelberg}, and bank-intervention models with impulse or semi-Markov capital injections \citep{yilun2020banksalvagecontrol,yilun2023capitalinjection}. Establishing convergence and computational gains for such outer algorithms is beyond the scope of this paper.}



\bibliographystyle{informs2014}
\bibliography{references}

\newpage

\begin{APPENDIX}{}

\section{ Proofs}\label{app_proofs}


\subsection{Proof of Lemma \ref{ratio}}

\begin{proof}{Proof}
By relaxing the constraints to $n,N\in\mathbb R^+$, the optimization problem
\begin{align*}
\underset{n> 0, N> 0}{\text{min}} \quad & \frac{1}{n}\sigma_{f-{g}}^2+\frac{1}{N}\sigma_{ g}^2. \\
\text{s.t.} \quad & nc_{f-g} + Nc_g \leq B,
\end{align*}
is jointly convex in $n,N\in\mathbb R^+$. The objective value also approaches infinity as $N\to0$ or $n\to0$. Consequently, the strict convexity ensures the existence of a global minimizer in the interior and directly solving the Lagrangian gives us the result.\Halmos
\end{proof}

\subsection{Proof of Theorem \ref{ratioofvariance}}
\begin{proof}{Proof}
   Based on the results of Lemma \ref{ratio}, we can deduce that this ratio again does not depend on the budget $B$, so we omit its dependence in the statement. Moreover, we can conveniently choose a budget of the form $B=n_0c_{f-g}+n_0\frac{\sigma_{g}}{\sigma_{f-g}}\cdot\sqrt{\frac{c_{f-g}}{c_{g}}}c_g$ which gives an allocation of $n=n_0$ and $N=n_0\frac{\sigma_{g}}{\sigma_{f-g}}\cdot\sqrt{\frac{c_{f-g}}{c_{g}}}$ for PEMC in accordance with Lemma \ref{ratio}, while comparing with $B/c_f$ samples for standard MC. The rest follows from Lemma \ref{pemcvariance} and the assumption $c_f = c_{f-g}$, i.e., $ \frac{\sigma_{f-g}^2}{\sigma_f^2}\bigg(1+\frac{\sigma_g}{\sigma_{f-g}}\cdot\sqrt{\frac{c_{g}}{c_{f}}}\bigg) + \frac{\sigma_{g}^2}{\sigma_f^2}\bigg(\frac{\sigma_{f-g}}{\sigma_{g}}\cdot\sqrt{\frac{c_{g}}{c_{f}}}+\frac{c_{g}}{c_{f}}\bigg) = \frac{\Big(\sigma_{f-g}+\sigma_g\sqrt{\frac{c_g}{c_f}}\Big)^2}{\sigma^2_f}$. \Halmos
\end{proof} 

\subsection{Proof of Theorem \ref{idealg}}
\begin{proof}{Proof}
    When $g = \mathbb{E}[f(\boldsymbol{Y}) \,|\,  \boldsymbol{X}]$ is the true regressor, i.e., $$g = \argmin_{h \text{ is $\mathcal F_X$ -measurable}}\mathbb E[(f(\boldsymbol{Y})-h(\boldsymbol{X}))^2],$$ and $f$ is square-integrable and $\sigma(X)$ is the sigma field or filtration generated by random variable $X$. It follows from the orthogonal projection theorem in Hilbert spaces that $\mathbb E[(f-g) h] =0$ for all square-integrable $\mathcal F_X$ -measurable $h$. Plugging in $h = g$, we obtain $\mathrm{Cov}(f-g, g)=0$. Consequently we have $\rho:=\mathrm{corr}(f, g)=\frac{\sigma_g}{\sigma_f}$ and $\sigma_f^2 = \sigma_{f-g}^2+\sigma^2_g$, which further implies $\sigma^2_g = \rho^2\sigma_f^2$ and $\sigma_{f-g}^2 = (1-\rho^2)\sigma^2_f$. The rest follows from $c_{f-g} = c_f$.\Halmos
\end{proof}

\subsection{Proof of Lemma \ref{learning_theory_summary}}\label{ltproofsection}
In this part, we go through the pipeline for a result in the form of Lemma \ref{learning_theory_summary} in Section \ref{LTE}. We begin by formally introducing the key concepts associated with the learning framework under consideration.

\begin{definition}
	Let $g_0$ denote the true regression function, defined by
	\begin{align*}
		g_0(\boldsymbol{\theta},\boldsymbol{X}(\boldsymbol{\theta}))
        :=& \mathbb{E}_{\text{risk neutral measure}(\boldsymbol{\theta})}\bigg[f_{\boldsymbol{\theta}_{\text{payoff}}}(\boldsymbol{Y}(\boldsymbol{\theta})) \,\bigg|\, (\boldsymbol{\theta}, \boldsymbol{X}(\boldsymbol{\theta}))\bigg]\\
		=& \argmin_{g \text{ measurable}} \mathbb{E}_{\boldsymbol{\theta}\sim\Theta, (\boldsymbol{X}(\boldsymbol{\theta}), \boldsymbol{Y}(\boldsymbol{\theta}))\sim\text{risk neutral measure}(\boldsymbol{\theta})}
		\biggl(f_{\boldsymbol{\theta}_{\text{payoff}}}(\boldsymbol{Y}(\boldsymbol{\theta})) - g(\boldsymbol{\theta}, \boldsymbol{X}(\boldsymbol{\theta}))\biggr)^2.
	\end{align*}
Note that equality exploits the well-known property of the squared loss minimizer: it is the conditional expectation of the target variable given the input features. The notable part is that, this formulation inherently incorporates the sampling over $\boldsymbol\theta \sim \boldsymbol\Theta$. This is because, in PEMC we generate training data by first drawing $\boldsymbol\theta$ from $\boldsymbol\Theta$ and then sampling $\boldsymbol{X}(\boldsymbol{\theta}), \boldsymbol{Y}(\boldsymbol{\theta})$ from the corresponding risk-neutral measure indexed by $\boldsymbol{\theta}$. However, since the result holds pointwise for each fixed $\boldsymbol\theta \in \boldsymbol\Theta$, it also holds in the aggregate setting where $\boldsymbol\theta$ is (uniformly) random. For simplicity, we do not delve into technical measurability considerations here.

Next, we introduce the hypothesis class and define the best-in-class predictor. The hypothesis class $\mathcal{G}$ plays a pivotal role in the learning framework, encapsulating the set of candidate functions from which the predictor $g$ is selected.

\begin{definition}

Let $\mathcal{G}$ be the hypothesis class induced by the NN model family. Define $g^*$ as the best-in-class function satisfying
	\begin{align*}
		g^*(\boldsymbol{\theta}, \boldsymbol{X}(\boldsymbol{\theta})):=& \argmin_{g \in \mathcal{G}} \mathbb{E}_{\boldsymbol{\theta}\sim\Theta, (\boldsymbol{X}(\boldsymbol{\theta}), \boldsymbol{Y}(\boldsymbol{\theta}))\sim\text{risk neutral measure}(\boldsymbol{\theta})}
		\biggl(f_{\boldsymbol{\theta}_{\text{payoff}}}(\boldsymbol{Y}(\boldsymbol{\theta})) - g(\boldsymbol{\theta}, \boldsymbol{X}(\boldsymbol{\theta}))\biggr)^2\\
		=& \argmin_{g \in \mathcal{G}} \mathbb{E}_{\boldsymbol{\theta}\sim\Theta, \boldsymbol{X}(\boldsymbol{\theta})\sim\text{risk neutral measure}(\boldsymbol{\theta})}
		\biggl((g_0 - g)(\boldsymbol{\theta}, \boldsymbol{X}(\boldsymbol{\theta}))\biggr)^2.
	\end{align*}
\end{definition}
where the second line follows again from the definition of $g_0$. Considering a model $g$ trained with $N_{\text{train}}$ samples and is held fixed during the evaluation of the expectation, as is standard when discussing generalization error, the approximation error is defined as
\begin{align*}
	\epsilon_{a}^{\mathcal{G}} := \mathbb{E}_{\boldsymbol{\theta}\sim\Theta, \boldsymbol{X}(\boldsymbol{\theta})\sim\text{risk neutral measure}(\boldsymbol{\theta})}
	\biggl((g_0 - g^*)(\boldsymbol{\theta}, \boldsymbol{X}(\boldsymbol{\theta}))\biggr)^2,
\end{align*}
and the statistical error from $g$ obtained from training on $N_{\text{train}}$ samples:
\begin{align*}
	\epsilon_{e}^{N_{\text{train}}} := \mathbb{E}_{ \boldsymbol{\theta}\sim\Theta, \boldsymbol{X}(\boldsymbol{\theta})\sim\text{risk neutral measure}(\boldsymbol{\theta})}
	\biggl((g^* - g)(\boldsymbol{\theta}, \boldsymbol{X}(\boldsymbol{\theta}))\biggr)^2.
\end{align*}
Finally, define the total error as
\begin{align*}
	\epsilon_{\text{total}} := \mathbb{E}_{\boldsymbol{\theta}\sim\Theta, \boldsymbol{X}(\boldsymbol{\theta})\sim\text{risk neutral measure}(\boldsymbol{\theta})}
	\biggl((g_0 - g)(\boldsymbol{\theta}, \boldsymbol{X}(\boldsymbol{\theta}))\biggr)^2.
\end{align*}
\end{definition}
It follows that
\begin{align}\label{bound}
	\epsilon_{\text{total}} \leq 2\epsilon_a^{\mathcal G} + 2\epsilon_e^{N_{\text{train}}}.
\end{align}
 The interaction between $\epsilon_{a}^{\mathcal{G}}$ and $\epsilon_{e}^{N_{\text{train}}}$ essentially captures the bias-variance tradeoff. By choosing more expressive model classes (e.g., richer neural network architectures) and increasing the training set size, we can jointly reduce $\epsilon_a^{\mathcal{G}}$ and $\epsilon_e^{N_{\text{train}}}$. Thus, under reasonable conditions, $\epsilon_{\text{total}}$ can be made arbitrarily small if both $\epsilon_a^{\mathcal{G}}$ and $\epsilon_e^{N_{\text{train}}}$. In the following sections, we analyze $\epsilon_a^{\mathcal{G}}$ and $\epsilon_e^{N_{\text{train}}}$ respectively.
\subsubsection{Approximation Error $\epsilon_a^{\mathcal{G}}$}\label{approximation_error}
Note that this best-in-class function $g^*$ depends on the distribution of $\boldsymbol{\theta}\sim\Theta$, meaning the notion of optimality is distribution-dependent on $\boldsymbol{\theta}$. This is also why we choose a distribution $\boldsymbol{\theta}\sim\boldsymbol{\Theta}$ that sufficiently covers the space. In this paper, $\mathcal{G}$ is induced by our choice of NN architecture and training procedure. The complexity of $\mathcal{G}$ significantly impacts both approximation and statistical errors. NNs are renowned for their flexibility and expressiveness, serving as universal approximators. According to the Universal Approximation Theorem \citep{cybenko1989approximation,hornik1991approximation}, NNs with a single hidden layer containing a sufficient number of neurons can approximate any continuous function on compact subsets of $\mathbb{R}^n$ to arbitrary accuracy. Modern extensions of this theorem provide more nuanced insights into how network depth and architecture influence approximation capabilities \citep{yarotsky2017error,lu2017expressive,shen2021neural, cybenko1989approximation,hornik1991approximation}. Typical universal universal approximation theorems deals with compact input space in $\mathbb R^n$ and point-wise convergence. In our context, we do not restrict $\boldsymbol{X}$ to live in a compact space and we only need $\mathscr L^2$ convergence, so the theorem from \cite{Pinkus1999} or \cite{LeshnoLinPinkusSchocken1993} is sufficient, which gives us
\begin{lemma}\label{uat}
    Suppose $g_0: \boldsymbol{\Theta}\times \mathbb R^{dim (\boldsymbol{X})}\to\mathbb R$ is in $\mathscr L^2(\mathbb P)$ where $\mathbb P$ is the probability measure that governs the joint distribution of $\boldsymbol{\theta}\sim\Theta, (\boldsymbol{X}(\boldsymbol{\theta}), \boldsymbol{Y}(\boldsymbol{\theta}))\sim\text{risk neutral measure}(\boldsymbol{\theta})$. Then, for any $\epsilon>0$, we can find a class of NN $\mathcal G$ such that 
    \begin{equation*}
        \epsilon_{a}^{\mathcal{G}} \leq \epsilon.
    \end{equation*}
\end{lemma}
\begin{proof}{Proof}
    See \cite{Pinkus1999} or \cite{LeshnoLinPinkusSchocken1993}. \Halmos
\end{proof}

\subsubsection{Statistical Error $\epsilon_{e}^{N_{\text{train}}} $}\label{statistical_error}

To quantify the capacity of $\mathcal{G}$, one could employ common complexity measures such as the Vapnik-Chervonenkis (VC) dimension \citep{vapnik1998statistical} or Rademacher complexity \citep{bartlett2002rademacher}. These measures provide bounds on the generalization error by capturing the richness of the hypothesis class. Modern studies have further refined our understanding of NN complexity. For instance, norm-based capacity controls \citep{bartlett2017spectrally}, PAC-Bayesian bounds \citep{dziugaite2017computing}, and local Rademacher complexities \citep{bartlett2005local} offer more nuanced insights. In particular, most neural network architectures are known to have finite complexity measures, such as finite VC dimension or other capacity metrics, when considered as a hypothesis class with a fixed number of parameters \citep{AnthonyBartlett1999}.

Consequently, when empirical risk minimization (ERM) is performed over these networks with controlled complexity, standard statistical learning theory guarantees apply, yielding generalization bounds that typically decrease on the order of $O(1/\sqrt{N_{\text{train}}})$ where $N_{\text{train}}$ is the sample size \citep{bartlett2017spectrally, Neyshabur2018, Arora2018} and the constant in the rate depends on the complexity measure. Our PEMC prediction model can be considered as a ERM estimator. Thus, these results ensure that, given sufficiently large training sets and appropriate capacity constraints (e.g., weight regularization or architectural choices), NNs can achieve low statistical error $\epsilon_{e}^{N_{\text{train}}}$.

Finally, there is an optimization error arising from the discrepancy between the empirical risk minimizer within $\mathcal{G}$ and the final trained model $g$. This error accounts for situations where the optimization algorithm does not perfectly identify the empirical risk minimizer. However, we do not consider this optimization error, effectively assuming the presence of an ideal “oracle” for optimization. Consequently, theoretical analyses typically treat our $g$ as readily available empirical minimizer.
\begin{lemma}\label{staterr}
    Suppose $g$ is the empirical risk minimizer of Algorithm \ref{sddd} and the neural network class $\mathcal G$ has a finite VC-dimension or Rademacher complexity. Then, for any $\epsilon>0$, there exists $N_{\text{train}}$ such that 
    \begin{equation*}
        \epsilon_{e}^{N_{\text{train}}} \leq \epsilon.
    \end{equation*}
\end{lemma}

\subsubsection{Technical Proof of Lemma \ref{learning_theory_summary}}\label{proofoflemma6}

Lemma \ref{uat} and Lemma \ref{staterr} allow us to control $\epsilon_{\text{total}}$ using \eqref{bound}, which allows us to prove Lemma \ref{learning_theory_summary} with the help of the following technical definition and lemma.
\begin{definition}
    Define $\epsilon_{\text{total}}(\boldsymbol{\theta}):=\mathbb E_{\boldsymbol{X}(\boldsymbol{\theta})\sim \text{risk neutral measure}(\boldsymbol{\theta})}\big(g(\boldsymbol{X}(\boldsymbol{\theta}))-g_0(\boldsymbol{X}(\boldsymbol{\theta}))\big)^2$.
\end{definition}
As a result, we have $\mathbb E_{\boldsymbol{\theta}\sim\boldsymbol{\Theta}}\epsilon_{\text{total}}(\boldsymbol{\theta}) = \epsilon_{\text{total}}$, which allows us to establish the following lemma. 
\begin{lemma}\label{diffvar}
	Let $G$ be a second moment bound for $g$, i.e.,
	\begin{align*}
		\mathbb{E}_{\boldsymbol{X}(\boldsymbol{\theta})\sim \text{risk neutral measure}(\boldsymbol{\theta})} g^2(\boldsymbol{X}(\boldsymbol{\theta})) \leq G.
	\end{align*}
	uniformly for all $\boldsymbol{\theta}\in\boldsymbol{\Theta}$. Then, for each $\boldsymbol{\theta}\in\boldsymbol{\Theta}$,
	\begin{align*}
		|\sigma_f^2 - \sigma_g^2 - \sigma_{f-g}^2| \leq 2\sqrt{\epsilon_{\text{total}}(\boldsymbol{\theta})}\sqrt{G}.
	\end{align*}
\end{lemma}

\begin{proof}{Proof}
    Fix any $\boldsymbol{\theta}\in\boldsymbol{\Theta}$. From the variance decomposition, we have
\begin{align*}
	\mathrm{Var}(f) = \mathrm{Var}(g) + \mathrm{Var}(f-g) + 2\,\mathrm{Cov}(f-g,g),
\end{align*}
which implies
\begin{align*}
	|\sigma_f^2 - \sigma_g^2 - \sigma_{f-g}^2| \leq 2|\mathrm{Cov}(f-g,g)|.
\end{align*}
Conditioning on $\boldsymbol{X}$ and utilizing the properties of $g_0$, we find that $\mathrm{Cov}(f-g,g) = \mathrm{Cov}(g_0 - g,g)$. Applying the Cauchy–Schwarz inequality yields the desired inequality.\Halmos
\end{proof}

We can now prove Lemma \ref{learning_theory_summary}.
\begin{proof}{Proof of Lemma \ref{learning_theory_summary}}
    Besides the assumptions in Lemma \ref{uat}, Lemma \ref{staterr} and Lemma \ref{diffvar}, we further assume
    \begin{equation}\label{upandlowerboundsecond}
        g\leq \mathbb{E}_{\boldsymbol{X}(\boldsymbol{\theta})\sim \text{risk neutral measure}(\boldsymbol{\theta})} g^2(\boldsymbol{X}(\boldsymbol{\theta})), \mathbb{E}_{\boldsymbol{Y}(\boldsymbol{\theta})\sim \text{risk neutral measure}(\boldsymbol{\theta})} f^2(\boldsymbol{Y}(\boldsymbol{\theta})) \leq G,
    \end{equation}
    for some $0<g<G<\infty$, uniformly for all $\theta\in\Theta$. Then, based on Lemma \ref{uat}, \ref{staterr} and Markov inequality, we can find $\mathcal G$ and $N_{\text{train}}$ such that, with probability at least $1-\delta$, the randomly sampled $\boldsymbol{\theta}\sim\boldsymbol{\Theta}$ satisfies
\begin{equation*}
    \epsilon_{\text{total}}(\boldsymbol{\theta})\leq \epsilon,
\end{equation*}
which, together with Lemma \ref{diffvar} gives
\begin{equation*}
        |\sigma_f^2 - \sigma_g^2 - \sigma_{f-g}^2| \leq O(\epsilon), 
    \end{equation*}
with the constant in $O$ not dependent on $\boldsymbol{\theta}$.
Then, for at least $1-\delta$ fraction of $\boldsymbol{\theta}\in\boldsymbol{\Theta}$, as one shrinks $\epsilon\to0$, $\epsilon_{\text{total}}(\boldsymbol{\theta})\leq \epsilon$ and $|\sigma_f^2 - \sigma_g^2 - \sigma_{f-g}^2| \leq O(\epsilon)$ implies $\sigma_g^2\to\sigma_{g_0}^2$ and $\sigma_{f-g_0}^2\to\sigma_{f-g}^2$ uniformly for all such $\boldsymbol{\theta}$. Since $\sigma_f^2 = \sigma_{g_0}^2 + \sigma_{f-g_0}^2$, this further implies $\frac{\sigma_g}{\sigma_f}\to\frac{\sigma_{g_0}}{\sigma_f}$ and $\frac{\sigma_{f-g}}{\sigma_f}\to\frac{\sigma_{f-g_0}}{\sigma_f}$, given the boundedness of \eqref{upandlowerboundsecond}. Thus, $\frac{\mathrm{Var}(PEMC)}{\mathrm{Var}(MC)}$ in \eqref{ration_funda} $\to$ $r(\rho,c)$ in \eqref{ratio_number}, or equivalently, with probability at least $1-\delta$, the randomly sampled $\boldsymbol{\theta}\sim\boldsymbol{\Theta}$ satisfies
\begin{equation*}
    \frac{\mathrm{Var}(\hat{\mu}_{PEMC})}{\mathrm{Var}(\hat{\mu}_{MC})}=r(\rho,c) + O(\epsilon),
\end{equation*}
where the constant in $O$ does not depend on $\boldsymbol{\theta}$.
This concludes the proof. \Halmos
\end{proof}

\subsubsection{Technical Proof of Theorem \ref{learningquality}}
When $\boldsymbol{\theta}$ is fixed and we assume
\begin{equation*}
        \mathbb{E}_{{X}(\boldsymbol{\theta})\sim\text{risk neutral measure}(\boldsymbol{\theta})}
	\biggl((g - g^*)(\boldsymbol{\theta}, {X}(\boldsymbol{\theta}))\biggr)^2 \leq \epsilon,
\end{equation*}
$\rho=corr(f,g^*) = \frac{\sigma_{g^*}}{\sigma_f}$ as shown in the proof of Theorem \ref{idealg}. Using the same techniques, we still have 
$\mathbb E(f-g^*)h = 0$ for all $\mathcal F_X$-measurable $h$. Plug in $h=g^*-g$, we have $\mathbb E(f-g^*)(g^*-g) = 0$ and consequently
\begin{equation*}
    \sigma^2_{f-g} = Var(f-g^*+g^*-g) = Var(f-g^*)+Var(g^*-g) \leq \sigma^2_{f-g^*} + \epsilon.
\end{equation*}

So we have $\frac{\sigma_{f-g}}{\sigma_f}\leq \sqrt{1-\rho^2+\frac{\epsilon}{\sigma^2_f}}$. On the other hand, we also have 
\begin{equation*}
    \sigma_g = \sqrt{Var(g^*+g-g^*)} \leq \sqrt{Var(g^*)} + \sqrt{Var(g-g^*)} \leq \rho\sigma_f+\sqrt{\epsilon}.
\end{equation*}
The rest of the proof follows from plugging Theorem \ref{ratioofvariance}.

\section{Numerical Results on Asian Option Pricing}\label{sec:empirical_evaluation_gbm}
We conduct an additional experiment on the arithmetic average Asian call option under the GBM model. In this example, there is a well-known CV constructed using the geometric average Asian option, whose closed-form mean can be obtained using the Black-Scholes formula. This provides us a benchmark to compare the performance of PEMC with the established CV method. However, it is important to note that this example is very special in that an effective CV with known mean is available. In all our real use cases in Section \ref{apps}, no effective CV has been known which more clearly motivates PEMC.



\subsection{Experimental Setup}

We implement the PEMC procedure detailed in Section~\ref{sec:illustration}, utilizing a NN as the predictive model. The asset price follows the GBM model:
\begin{align}\label{gbm}
dS_t =& r S_t dt + \sigma S_t dW_t^S.
\end{align}
The payoff calculation $P_A(\{S_t\}_t) = (\frac{1}{n_D}\sum_{i=1}^{n_D} S_{t_i} - K, 0)^+$ with $n_D=252$, where simulations are carried out using daily time increments (i.e., $\Delta t=1$). The parameter space $\boldsymbol{\theta}:=(r,S_0,\sigma,K)\in\boldsymbol{\Theta}$ is: $r \in [0.01, 0.03]$, $S_0 \in [80, 120]$, $\sigma \in [0.05, 0.25]$, and $K \in [90, 110]$. For our feature $\boldsymbol{X}$, we experiment with the whole sum of Brownian increments $W_T^S:=\sum_{j=1}^{252}\Delta W^S_{j}$ with $\dim(\boldsymbol{X})=1$. We also tried a more granular representation where we partition 252 Brownian increments evenly into 14 parts, each summing 18 consecutive increments: $[\boldsymbol X]_i = \sum_{j=18(i-1)+1}^{18i}\Delta W^S_j$ for $i\in [14]$, yielding $\dim(\boldsymbol{X})=14$. The NN is trained on a dataset comprising $N_{\text{train}} = 1.28 \times 10^6$ samples, generated by uniformly sampling $\boldsymbol{\theta}\sim\boldsymbol{\Theta}$.

The NN architecture consists of three components: a $\boldsymbol{\theta}$ network branch, a $\boldsymbol{X}$ network branch, and a combined network. This architecture incorporates several modern deep learning practices, including Batch Normalization \citep{ioffe2015batch}, Rectified Linear Unit (ReLU) activations \citep{nair2010rectified}, and skip connections \citep{he2016deep} to enhance training stability and model performance. The $\boldsymbol{\theta}$ network branch processes the 4-dimensional input $\boldsymbol{\theta}$ using two fully connected layers with Batch Normalization and ReLU activation functions, outputting a 10-dimensional feature vector for subsequent processing. The $\boldsymbol{X}$ network branch uses two fully connected layers with width $\max(32, 2\dim(\boldsymbol{X}))$ neurons. Finally, the combined network integrates outputs from both branches through concatenation and employs skip connections featured in ResNet architectures to preserve information flow. The concatenated features pass through two additional fully connected layers with Batch Normalization and ReLU activation to produce the final prediction for PEMC. The network is implemented using PyTorch \citep{paszke2019pytorch} and trained using the Adam optimizer \citep{kingma2014adam} with a learning rate of $1 \times 10^{-3}$, incorporating dropout layers \citep{srivastava2014dropout} with a rate of 0.5 after each hidden layer to prevent overfitting. 
\subsection{Performance Comparisons}
For comparison, we implement PEMC, MC and the classical CV method with geometric Asian option $$P_G(\{S_t\}_t) = ((\prod_{i=1}^{n_D}S_{t_i})^{\frac{1}{{n_D}}} - K, 0)^+,$$ for the evaluation at $\boldsymbol{\theta}=(r,S_0,\sigma,K)=(0.02,100,0.2,100)$. The classical CV estimator using the geometric Asian option as CV (which we will now simply call ``geometric CV estimator" for convenience) is given by $$\text{CV} = \frac{1}{n}\sum_{i=1}^n \left( P_A(\{S_t^{(i)}\}_t) -  P_G(\{S_t^{(i)}\}_t)  \right)+P_G^{\text{exact}}(\boldsymbol{\theta}),$$ where $P_G^{\text{exact}}(\boldsymbol{\theta})$ is the closed-form price of geometric Asian option (with correction regarding $n_D$) \citep{glasserman2013monte}. We first evaluate the mean of $2\times 10^{9}$ MC samples as the ground truth value $A_0$. Then, we evaluate $n$ sample average of the MC, PEMC ($N=10n$) and the geometric CV estimator for $n=1000,4000,9000$ respectively, to record 3 different estimator for $A_0$. After that, we repeat the experiments 300 times to get 300 estimators for each $n\in\{1000,4000,9000\}$ and each method in $\{\text{MC, PEMC, Geometric CV}\}$, to compare against $A_0$. The performance of the estimators is summarized in Table \ref{tab_mse_comparison} and Figure \ref{fig2}.

\begin{table}[h]
    \centering
    \caption{Root Mean Squared Error from 300 Experiments}
    \begin{tabular}{lccc}
        \hline
        \textbf{Method} & $n=1000$ & $n=4000$ & $n=9000$ \\
        \hline
        Standard MC & 0.2376 & 0.1173 & 0.0854 \\
        PEMC ($\dim \boldsymbol X =1$) & 0.1509 & 0.0809 & 0.0481 \\
        PEMC ($\dim \boldsymbol X =14$) & 0.0781 & 0.0397 & 0.0261 \\
        Geometric CV & 0.0099 & 0.0051 & 0.0036 \\
        \hline
    \end{tabular}
    \label{tab_mse_comparison}
\end{table}

\begin{figure}
     \FIGURE
   {\includegraphics[scale=0.4]{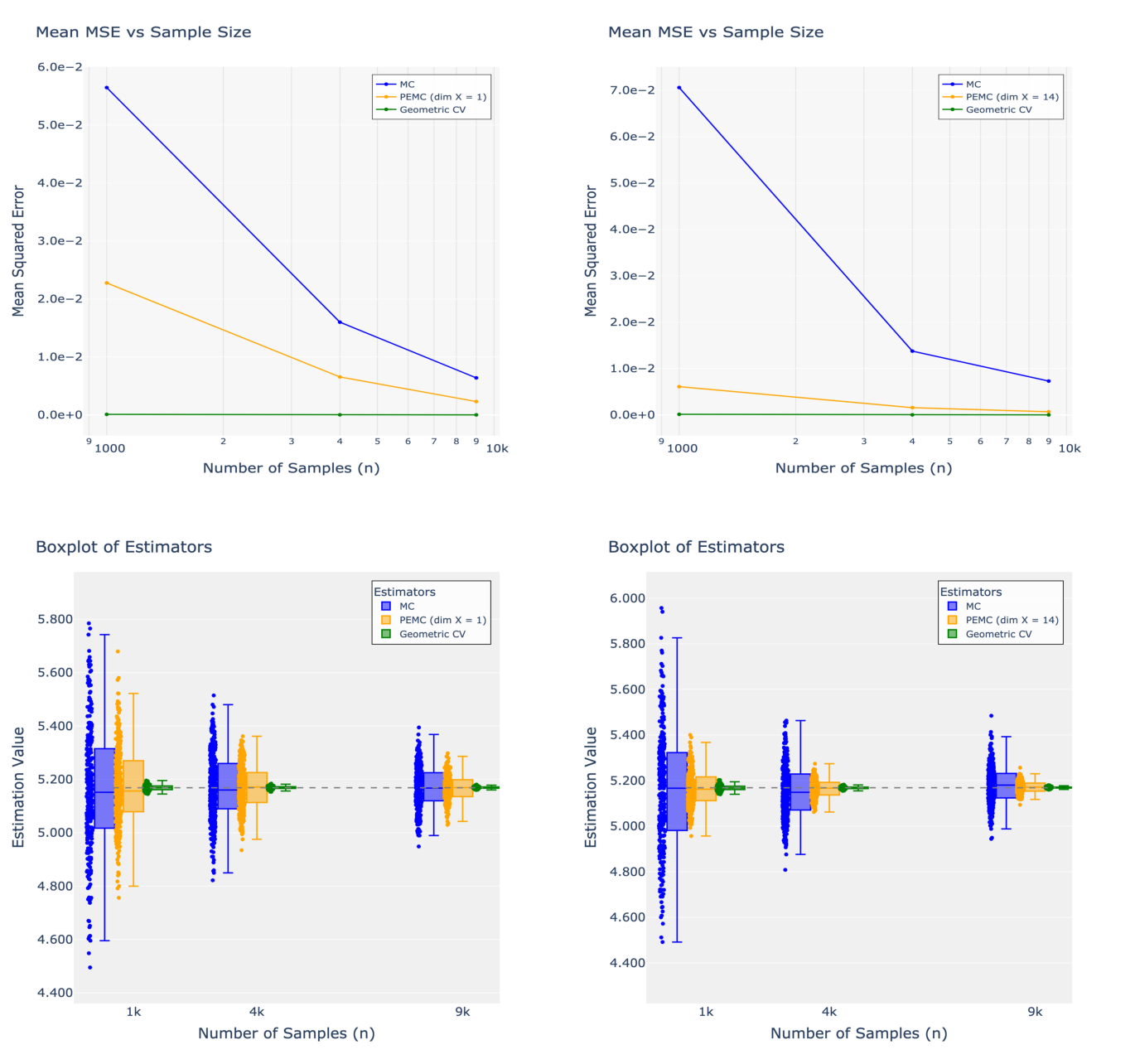}}  
   {Performance of Estimators for Asian Options. \label{fig2}}
{Comparison of estimator across MC, PEMC, and Geometric CV based on 300 experiments. Top: MSE plot as a function of $n$ for (left) PEMC with $\dim \boldsymbol X=1$ and (right) PEMC with $\dim \boldsymbol X=14$, both compared against MC and Geometric CV. Bottom: Corresponding boxplots of the 300 estimates. }
\end{figure}

Given that all estimators are unbiased, our analysis focuses exclusively on variance-related metrics. The experimental results demonstrate the superior performance of the PEMC framework compared to the standard MC estimator. In particular, in all sample sizes, the PEMC framework delivers visible variance reduction over MC: PEMC's basic variant with $\dim \boldsymbol X=1$ achieves a 30-40\% reduction in root mean squared error (RMSE) relative to MC, while the more sophisticated variant with $\dim \boldsymbol X=14$ attains an impressive 65-70\% reduction in RMSE over MC. However, the geometric CV estimator emerges as the most efficient estimator by far, with RMSE an order of magnitude smaller than MC.


\subsection{Using PEMC as a ``Boost" to Known CV}
Lastly, we demonstrate how PEMC can be used to improve existing CV. This supports the creative use of PEMC for further efficiency enhancement even for problems that already possess some variance reduction approaches. In particular, when a CV with known mean is available, such as in this current example, PEMC can be built upon it to potentially achieve even greater variance reduction. To see how, in Algorithm \ref{sddd}, we change the label from $P_A(\{S_t\}_t)$ (MC baseline) to $P_A(\{S_t\}_t)-P_G(\{S_t\}_t)+P_G^{\text{exact}}(\boldsymbol{\theta})$ (Geometric CV baseline), and let the model $g$ learn

\begin{samepage}
\begin{align}\label{boost_train}
    g(\boldsymbol{\theta}, \boldsymbol{X}(\boldsymbol{\theta}))\approx&\mathbb E_{\text{risk neutral measure}(\boldsymbol{\theta})}\bigg[P_A(\{S_t\}_t)-P_G(\{S_t\}_t)+P_G^{\text{exact}}(\boldsymbol{\theta})\bigg|(\boldsymbol{\theta}, \boldsymbol{X}(\boldsymbol{\theta}))\bigg]\nonumber\\
    =&\mathbb E_{\text{risk neutral measure}(\boldsymbol{\theta})}\bigg[P_A(\{S_t\}_t)-P_G(\{S_t\}_t)\bigg|(\boldsymbol{\theta}, \boldsymbol{X}(\boldsymbol{\theta}))\bigg]+P_G^{\text{exact}}(\boldsymbol{\theta}).
\end{align}
\end{samepage}

This formulation suggests we can train $g$ on the difference of $P_A(\{S_t\}_t)-P_G(\{S_t\}_t)$ as labels and then add back the closed-form $P_G^{\text{exact}}(\boldsymbol{\theta})$ to $g$. We implemented this approach, which we call ``Boost PEMC," using the basic PEMC variant where $\boldsymbol{X} = W_T^S:=\sum_{j=1}^{252}\Delta W^S_{j}$ with $\dim \boldsymbol{X}=1$. As shown in Table \ref{tab_mse_comparison_boost} and Figure \ref{figboost}, this Boost PEMC indeed achieves variance reduction over the sophisticated Geometric CV estimator, delivering a 35-40\% reduction in RMSE, remarkably consistent with the relative performance gain we observed when applying PEMC to the MC baseline in Figure \ref{fig2}.
\begin{table}[h]
    \centering
    \caption{Root Mean Squared Error from 300 Experiments}
    \begin{tabular}{lccc}
        \hline
        \textbf{Method} & $n=1000$ & $n=4000$ & $n=9000$ \\
        \hline
        
        Boost PEMC ($\dim \boldsymbol X =1$) & 0.0065 & 0.0031 & 0.0021 \\
       
        Geometric CV & 0.0099 & 0.0051 & 0.0036 \\
        \hline
    \end{tabular}
    \label{tab_mse_comparison_boost}
\end{table}
\begin{figure}
     \FIGURE
   {\includegraphics[scale=0.4]{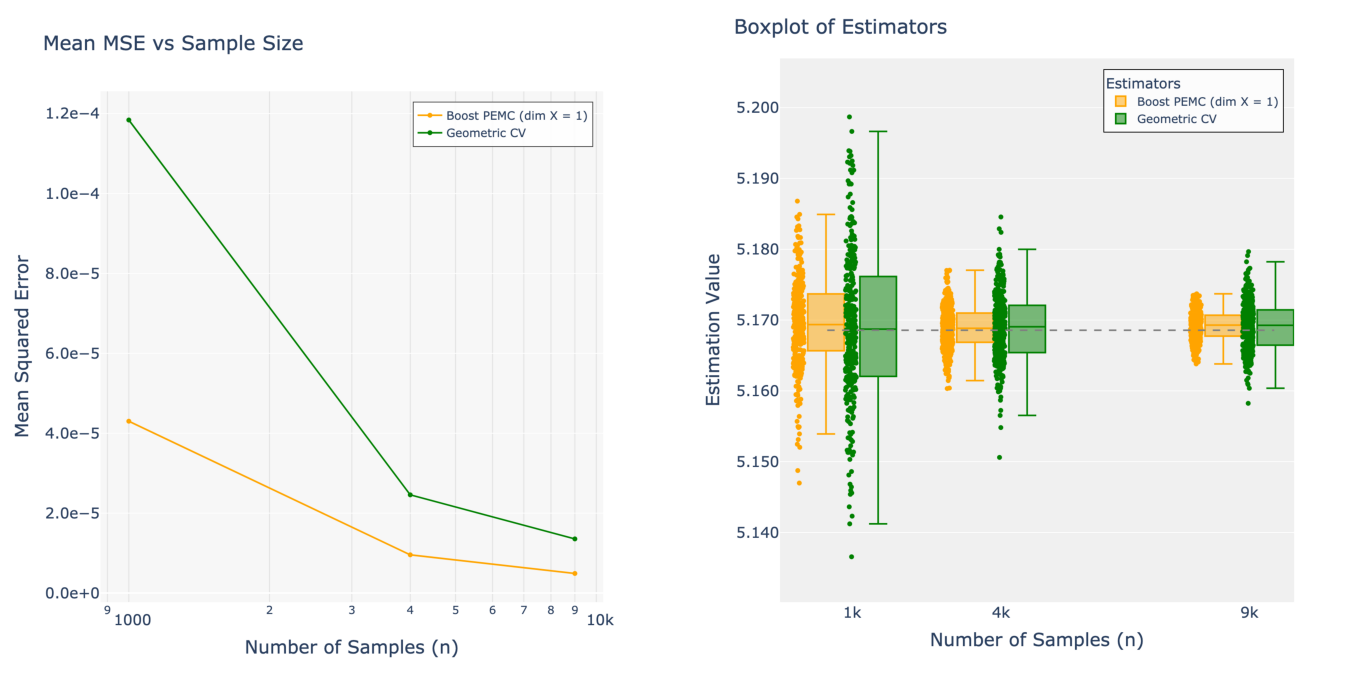}}  
   {Performance of CV-based Estimators for Asian Options. \label{figboost}}
{Comparison of estimator performance for Boost PEMC and Geometric CV based on 300 experiments. Left: MSE plot as a function of $n$. Right: Corresponding boxplots of 300 estimates. }
\end{figure}


\section{Additional Details on the Ambulance Diversion Example}\label{amb_div}
We provide more details on the ambulance diversion policy evaluation example in Section \ref{sec:EMS}, where we optimize the diversion threshold $\tau$ to minimize mortality. Here we detail the simulation's distributional assumptions. Patient arrivals are modeled by a nonhomogeneous Poisson process, $\lambda(t)$, calibrated to historical data from the pandemic's peak. Upon arrival, each patient is assigned a triage level $\ell \in \{1,\dots, 5\}$ via a multinomial distribution that shifts according to COVID surges. Service times in the emergency department follow exponential distributions with rate $\mu_\ell$. At each event, if the queue length at a chosen hospital exceeds a threshold, the ambulance is diverted to an alternative facility, trading off increased travel against avoiding overcrowding. Specifically, our discrete-event simulation (spanning a one-week horizon) envisions two hospitals (A and B) under nonhomogeneous Poisson arrivals—with ``normal" or ``crisis" modes dictating time-varying rates—and assigns each patient a triage level from ${1,2,3,4,5}$. Service rates $(\mu_{\ell}^A,\mu_{\ell}^B)$ differ by hospital and triage $\ell$, while death times follow a Weibull distribution parameterized by $(\alpha,\beta)$.

Let $\boldsymbol{\lambda} = (\lambda_{h,d})$ denote the matrix of arrival rates across hours and days. Patients arrive with triage levels $L\in\{1,\dots,5\}$ (where 1 is most critical), drawn from multinomial probabilities $\{p_\ell, p_\ell^\text{crisis}\}$. Within each ED, service times for triage $\ell$ follow exponential distributions $\exp(\mu_{\ell})$, and mortality follows a logistic model: $\mathbb{P}(\text{death} \mid L=\ell, W=w) = \frac{1}{1 + \exp(-(\beta_{0,\ell} + \beta_{1,\ell}w))} $. Notably, 25\% of arrivals come by ambulance. If an ED's queue length $Q$ exceeds a threshold $\tau$, any new ambulances divert to the other ED, incurring extra travel time. In a data-adaptive version, $\tau$ is updated each shift using predicted arrival rates. All model inputs—arrival rates, triage probabilities, and service rates—come from empirical data \cite{Wrede2020}, while the logistic mortality parameters $\beta_{0,\ell}$, $\beta_{1,\ell}$ are tuned (with minor per-patient perturbations) to capture heterogeneous risk profiles.

In this example, we construct the feature of \[
\begin{aligned}
(\textcolor{blue}{\tau},\ &\ 
\textcolor{blue}{\texttt{hosp2\_doctor\_shift\_counts}},\ 
\textcolor{blue}{\texttt{crisis\_factor}},\ 
\textcolor{purple}{\texttt{hopsital1\_max\_patience}}, \\
&\ \textcolor{purple}{\texttt{hopsital2\_max\_patience}},\ 
\textcolor{purple}{\texttt{total\_number\_of\_patients}},\ 
\textcolor{purple}{\texttt{total\_service\_time}}, \\
&\ \textcolor{purple}{\texttt{max\_patience\_time}},\ 
\textcolor{purple}{\texttt{total\_death\_time}}, \\
&\ \textcolor{purple}{\texttt{total\_number\_of\_patients\_without\_life\_threatening\_symptoms}}, \\
&\ \textcolor{purple}{\texttt{triage1\_count}},\ 
\textcolor{purple}{\texttt{triage2\_count}},\ 
\textcolor{purple}{\texttt{triage3\_count}},\ 
\textcolor{purple}{\texttt{triage4\_count}},\ 
\textcolor{purple}{\texttt{triage5\_count}}).
\end{aligned}
\]

The \textcolor{blue}{blue} features are $\boldsymbol{\theta}$ while the \textcolor{purple}{purple} ones are $\boldsymbol{X}$. The simulation starts by first sampling the number of patients, which follows a Poisson distribution with parameter of $\lambda T$, denoted by $N_p$. \texttt{hopsital1\_max\_patience} and \texttt{hopsital2\_max\_patience} can be simulated with max of weilbull distribution, and \texttt{max\_patience} can be computed by taking the max of these quantities (which has closed-form density function, i.e., maximum of Weibull's). \texttt{total\_service\_time} can be simulated by first sampling service time for $N_p$ times with $\exp(\lambda_{\text{service}})$ and summing up (which has closed-form density function, i.e., sum of exponential is Gamma ). \texttt{total\_death\_time} can be simulated by taking first sampling individual using inversion sampling from $F^{-1}(u)$, where $F^{-1}:$

\[
\begin{aligned}
F^{-1}_{t, B, \nu}(u) = 
\begin{cases} 
    0, & \text{if } u \leq \text{mortality\_func}(0, t, B, \nu) \\[10pt]
    \infty, & \text{if } u \geq K[t] \\[10pt]
    -\frac{\ln \left( \left( \frac{K[t] - A[t]}{u - A[t]} \right)^{(\nu + 0.25t)} - 1 \right)}{3t (B + 5 - t)}
    + \frac{(-2 + 2.5a)t}{(B + 5 - t)}, & \text{otherwise}.
\end{cases}
\end{aligned}
\]

\[
K = \{1: 1.0, 2: 0.9, 3: 0.05, 4: 0.02, 5: 0.01\}
\]
\[
A = \{1: 0.6, 2: 0.1, 3: 0.0, 4: 0.0, 5: 0.0\}
\],

\[
\text{mortality\_func}(x, a, t, B, \nu) =
A[t] + \frac{K[t] - A[t]}{\left( 1 + (3t) \cdot \exp \left( -(B + 5 - t)x + (-2+2.5a) t \right) \right)^{\frac{1}{\nu + 0.25t}}}
\]
where $B \sim \text{Unif}(2.5, 3.5)$, $\nu \sim \text{Unif}(1.5, 2.5)$, and the keys of $K$ and $A$ are the triage levels. Triage count of each level can be simulated by firstly sampling from multi-nominal distribution with probabilities from triage 1 to 5 as $\{0.1098, 0.2761, 0.4596, 0.1297, 0.0248\}$, and then counting the number of each category.

The quantity we are estimating is the mortality. The PEMC predictor is a random forest. The amount of training data is $10^5$. During the pre-training stage, we collect the training data by the sampling procedure in Table \ref{tab:training_data}.

\begin{table}[h]
    \centering
     \caption{Training Data Sampling Distributions}
    \label{tab:training_data}
    \begin{tabular}{|l|c|}
        \hline
        \textbf{Variable} & \textbf{Distribution} \\ 
        \hline
        $\tau$ & $\text{Unif}\{0, 50\}$ \\
        \texttt{hosp2\_doctor\_shift\_counts\_1} & $\text{Unif}\{1, 3\}$ \\
        \texttt{hosp2\_doctor\_shift\_counts\_2} & $\text{Unif}\{1, 3\}$ \\
        \texttt{hosp2\_doctor\_shift\_counts\_3} & $\text{Unif}\{2, 6\}$ \\
        \texttt{hosp2\_doctor\_shift\_counts\_4} & $\text{Unif}\{1, 4\}$ \\
        \texttt{hosp2\_doctor\_shift\_counts\_5} & $\text{Unif}\{2, 5\}$ \\
        \texttt{hosp2\_doctor\_shift\_counts\_6} & $\text{Unif}\{1, 3\}$ \\
        \texttt{crisis} & $\text{Unif}[1, 2]$ \\
        \hline
    \end{tabular}
\end{table}

During the evaluating stage, we fix \texttt{ciris} $= 1.25$, \texttt{hosp2\_doctor\_shift\_counts} $= [2, 2, 4, 2, 4, 1]$ and a set of thresholds in $\{1,...,40\}$.

\begin{table}[ht]
\centering
\caption{\small{Error analysis for mortality estimation with varied thresholds.}}
\label{tab:example}
\scalebox{1.0}{%
\begin{tabular}{c c c c c}
\toprule
\textbf{Threshold} & \multicolumn{2}{c}{\textbf{MSE} $\downarrow$} & \multicolumn{2}{c}{\textbf{MAE} $\downarrow$} \\
 & MC & PEMC & MC & PEMC \\
\midrule
0  & 10.974 & 4.142 & 2.765 & 1.587 \\
4  & 10.749 & 3.899 & 2.751 & 1.500 \\
8  & 10.882 & 4.080 & 2.786 & 1.554 \\
12 & 10.243 & 3.609 & 2.638 & 1.466 \\
16 & 9.777 & 3.326 & 2.555 & 1.422 \\
20 & 9.535 & 3.266 & 2.497 & 1.416 \\
24 & 9.291 & 3.181 & 2.456 & 1.394 \\
28 & 9.236 & 3.192 & 2.445 & 1.403 \\
32 & 9.212 & 3.145 & 2.434 & 1.392 \\
36 & 9.175 & 3.120 & 2.427 & 1.382 \\
40 & 9.174 & 3.122 & 2.428 & 1.376 \\
\bottomrule
\end{tabular}
}
\end{table}

\section{Further Implementation Details of PEMC}\label{xvastuff}
We discuss several further implementation details that are important for successfully applying PEMC in practice.

\subsection{Parameter Space $\Theta$}
When determining the parameter space $\boldsymbol{\Theta}$, a key consideration is how frequently the model will be updated. To ensure the performance of PEMC, it is desirable that the training data encompass all practical scenarios the model is expected to encounter within the designated update period. For example, if one is calibrating a Heston model \eqref{heston} for certain SPY ETF, on a given trading day, the calibration yields a specific set of parameters: $(S_0, \nu_0, r, \eta,\delta,\rho, \kappa)=(520,20\%,4.5\%,0.04, 0.3,-0.7,0.2)$ and the user of PEMC chooses to update the prediction model in PEMC once a month, then, based on historical data and market conditions, financial practitioners could make reasonable guesses or confidence intervals about the bounds within which these parameters are likely to fluctuate over the coming month. 

This process is conceptually similar to the selection of uncertainty sets in robust optimization \citep{ben2009robust} or distributionally robust optimization \citep{esfahani2018data}, where the goal is to include realizations of parameters with high likelihood. However, unlike robust optimization approaches, which often favor data-dependent uncertainty sets, our focus here is also guided by the expertise of financial engineers. The training parameter space is designed to reflect realistic and practical scenarios derived from domain knowledge, rather than being strictly driven by statistical guarantees. In applications, we construct reasonable $\boldsymbol\Theta$ for practical models, such as forward curves in HJM modeling \citep{glasserman2013monte} or local stochastic volatility 2D grids \citep{gatheral2006volatility}, among others.

\subsection{Eliminating Data Storage Overheads
}
In typical ML workflows, training datasets—often costly, scarce, and carefully curated—are stored and reused extensively. By contrast, within the PEMC framework, the nature of data generation during the prediction model training in Algorithm \ref{sddd}, actually does not require extensive data storage. This difference stems from two main considerations. First, the volume of training data required to achieve a well-performing prediction model can be large. Storing all of it would be both expensive and unnecessary. In PEMC, the training data can be produced directly via MC simulation, ensuring an effectively unlimited supply. Second, this flexibility allows for a more efficient workflow. Data can be generated “on the fly” and processed in streaming fashion. For example, to train on a large number of $N_{\text{train}}$ samples, one could iteratively produce small batches, train the model on these batches, and discard them after training. In this manner, we have successfully trained models in the applications using data on the order of $10^7$--$10^8$.

\subsection{Evaluation Metrics}
While training the NN estimator $g$ using MSE loss is common, it is not always clear how to interpret the resulting MSE score. Unlike some well-established benchmarks (e.g., classification accuracy), there is no canonical threshold or known “good” MSE value for a given problem. This ambiguity makes it challenging to determine when the network is sufficiently trained. To address this, we can exploit the fact that $g$ is meant to represent the conditional expectation $g=\mathbb{E}[f\mid\text{input}]$. If the network approximates this expectation well, then the sample average of $g(\boldsymbol{X})$ should be close to the sample average of $f(\boldsymbol{Y})$ over a given dataset. One practical diagnostic is to compute the Mean Absolute Relative Error (MARE) between these two averages. If $\mathbb E[g(\boldsymbol{X})] \approx \mathbb E[f(\boldsymbol{Y})]$, it provides a tangible indication that $g$ is capturing the underlying expectation. Our empirical experience suggests this criterion is very effective in practice (a 5-1\% MARE typically indicates exceptional PEMC),  complementing common techniques like early stopping in ML workflows. More importantly, as the we have shown in the theory, if $g$ closely approximates the conditional expectation, the variance reduction in PEMC is guaranteed—even if marginally—relative to standard MC.

\section{Additional Details on Hyper-Parameters for Variance Swaps}\label{Experiment_tab_graph} Tables \ref{tab:nn_params} and \ref{tab:hyper-param-swaption} detail the hyper-parameter setups for the NNs presented in Section \ref{sec:var swap}. 

\begin{table}[h]
    \centering
    \caption{Neural Network Architecture Parameters}
    \begin{tabular}{l|l}
        \hline
        \textbf{CNN Branch} & \textbf{Feed-forward Branch} \\
        \hline
        \begin{tabular}[t]{@{}l@{}}
            kernel size: 3, stride: 1, padding: 1 \\
            max pooling (kernel: 2, stride: 2, padding: 0)
        \end{tabular} & 
        \begin{tabular}[t]{@{}l@{}}
            hidden dim: 512 \\
            output dim: 128
        \end{tabular} \\
        \hline
    \end{tabular}
    \label{tab:nn_params}
\end{table}

\begin{table}[h]
\centering
\caption{Hyper-parameter setup for the neural network}
\label{tab:hyper-param-swaption}
\renewcommand{\arraystretch}{1.2} 
\resizebox{0.9\textwidth}{!}{
\begin{tabular}{p{5cm}|p{5cm}|p{4cm}} 
\hline
\textbf{2D function branch parameters} & 
\textbf{1D function branch parameters} &
\textbf{Vector feature branch parameters} \\ \hline
\begin{tabular}{@{}l@{}} 
Kernel size: (1, 3) \\ 
Stride: (1, 3) \\ 
Padding: 0 \\ 
AvgPool2d kernel size: (2, 2) \\ 
AvgPool2d stride: (2, 2) \\ 
AvgPool2d padding: 0 
\end{tabular} & 
\begin{tabular}{@{}l@{}} 
Kernel size: 10 \\ 
Stride: 3 \\ 
Padding: 0 \\ 
AvgPool1d kernel size: 2 \\ 
AvgPool1d stride: 2 \\ 
AvgPool1d padding: 0 
\end{tabular} & 
\begin{tabular}{@{}l@{}} 
Input dim: 9 \\
Hidden dim: 512 \\ 
Output dim: 128 
\end{tabular} \\ \hline
\multicolumn{3}{c}{\textbf{Feed-forward Synthesizer Parameters}} \\ \hline
\multicolumn{3}{c}{
\begin{tabular}{@{}l@{}} 
Hidden dim: 128 \\ 
Output dim: 1 
\end{tabular}
}
\\ \hline
\end{tabular}
}
\end{table}

\section{Review of HJM }\label{hjmappendix}
The HJM model directly describes the evolution of the entire forward rate curve, offering greater flexibility than traditional short-rate models. For illustration purposes, we focus on a one-factor specification with exponential volatility structure \citep{glasserman2013monte}, though the framework readily extends to multi-factor cases. 

Consider the problem of pricing a swaption. This is a contract granting its holder the right, but not the obligation, to enter into an interest rate swap at a future date. In a standard interest rate swap, one party agrees to pay a fixed rate while receiving a floating rate, and the other party does the opposite. Consider a swap with \( n_p \) fixed payment periods, each of length \(\Delta t'\), starting at time \( t'_0 \) and ending at time \( t'_{n_p} = t'_0 + \sum_{l=1}^{n_p} \Delta t' \). The value of this swap at time \( t'_0 \) is:
\[
V_{t'_0} = C \left( R \sum_{l=1}^{n_p} B(t'_0, t'_l) \Delta t' + B(t'_0, t'_{n_p}) - 1 \right),
\]
where \( C \) is the notional amount (contract size), \( R \) is the contractual fixed rate, or equivalently the swaption strike, and \( B(t'_0, t'_l) \) is the discount factor from \( t'_0 \) to \( t'_l \). A swaption provides the holder with the option to enter into this swap at \( t'_0 \).
The payoff of the swaption is simply
$$\max(0, V_{t'_0}),$$
and its expectation under the risk-neutral measure gives the price of swaptions. To specify the risk neutral measure, one needs the forward rate process needed to price the bond. The time $t$ price of a zero-coupon bond \( B(t,T) \) maturing at time \( T \) is given by the negative exponential of the cumulative forward rate \( f(t,u) \) as:
$B(t,T) = \exp\left(-\int_t^T f(t,u)\,du\right)$, or equivalently $\frac{\partial \log B(t,T)}{\partial T} = -f(t,T).$ The HJM framework \citep{glasserman2013monte} models the dynamics of forward rate curve directly:
\[
df(t,T) = \mu(t,T) \, dt + \sigma(t,T)^\top dW(t),
\]
where \(\mu(t,T)\) is the drift, \(\sigma(t,T)\) is the volatility function of the forward rate, and \( W(t) \) is a Brownian motion. In contrast to short-rate models (e.g.,  \cite{vasicek1977equilibrium} or \cite{cox1985theory}), which only model the dynamics of the short-term interest rate, the HJM model directly models the dynamics of the entire term structure of interest rates \citep{heath1992bond}. The HJM model is widely used in practice because of its flexibility in modeling interest rate derivatives like swaptions and its ability to incorporate complex volatility structures \citep{brigo2006interest}. However, the model's generality also leads to the need for sophisticated numerical methods for simulation \citep{glasserman2013monte}. A key property of the HJM model is the no-arbitrage condition \citep{glasserman2013monte}, which specifies the drift completely by the volatility:
\begin{equation}\label{noarb}
\mu(t,T) = \sigma(t,T)^\top \int_t^T \sigma(t,u) \, du.
\end{equation}
Thus, in the HJM framework, the model is fully specified by defining the initial forward rate curve $f(0,T)$ and the structure of the volatility $\sigma(t,T)$. In our experiment we used a simple one factor HJM for illustration.

\end{APPENDIX}

\end{document}